\documentclass{article} % For LaTeX2e
\usepackage[pagebackref=false,breaklinks=true,letterpaper=true,colorlinks=true,citecolor=blue,bookmarks=false]{hyperref}

\usepackage{iclr2019_conference,times}

\usepackage{microtype}
\usepackage{graphicx}
\usepackage{subfigure}
\usepackage{booktabs} % for professional tables
\usepackage{epsfig}
\usepackage{amsmath}
\usepackage{amssymb}
\usepackage{amsfonts}
\usepackage{dsfont}
\usepackage[utf8]{inputenc}
\usepackage{multirow}
\usepackage{tabulary}
\usepackage{rotating}
\usepackage{bbding}
\usepackage{bm}
\usepackage{verbatim}
\usepackage{wrapfig}
\usepackage{longtable}
\usepackage{amsthm}

\newtheorem{prop}{Proposition}

\newcommand{\eg} {\emph{e.g. }}
\newcommand{\ie} {\emph{i.e. }}
\newcommand{\vs} {\emph{v.s. }}
\newcommand*{\tran}{{^{\mkern-1.5mu\mathsf{T}}}}
\newcommand{\w} {{ \mathbf{w} }}

\newcommand{\x} {{ \mathbf{x} }}

\newcommand{\sw} {{ \tilde{\w} }}
\newcommand{\swj} {{ {\tilde{\w}^j} }}
\newcommand{\ord} {{ {s\hspace{-1pt}g\hspace{-1pt}d} }}
\newcommand{\bn} {{ {b\hspace{-1pt}n} }}
\newcommand{\wn} {{ {w\hspace{-1pt}n} }}

\newcommand{\hh} {{ \hat{h} }}

\newcommand{\eff} {{ \mathrm{eff} }}

\newtheorem{innercustomthm}{Theorem}
\newenvironment{customthm}[1]
{\renewcommand\theinnercustomthm{#1}\innercustomthm}
{\endinnercustomthm}

\title{Towards Understanding Regularization in \\Batch Normalization }

\author{
Ping Luo$^{1,3}$\thanks{The first three authors contribute equally. Corresponding to pluo.lhi@gmail.com, \{wangxinjiang, pengzhanglin\}@sensetime.com, weqish@link.cuhk.edu.hk.}\hspace{15pt}Xinjiang Wang$^{2\ast}$\hspace{15pt}Wenqi Shao$^{1\ast}$\hspace{15pt}Zhanglin Peng$^2$\\
$^1$The Chinese University of Hong Kong\hspace{15pt}$^2$SenseTime Research\hspace{15pt}$^3$The University of Hong Kong\\
}

\iclrfinalcopy % Uncomment for camera-ready version, but NOT for submission.
\begin{document}

\maketitle
\vspace{-15pt}

\begin{abstract}
\vspace{-5pt}
Batch Normalization (BN)
%makes output of hidden neuron had zero mean and unit variance,
improves both convergence and generalization in training neural networks.
% by regularizing
This work understands these phenomena theoretically.
We analyze BN by using a basic block of neural networks, consisting of a kernel layer, a BN layer, and a nonlinear activation function.
This basic network helps us understand the impacts of BN
%, where the results are generalized to deep models in numerical studies.
%
%We explore BN
in three aspects.
First, by viewing BN as an implicit regularizer, BN can be decomposed into population normalization (PN) and gamma decay as an explicit regularization.
%an analytical form of its regularization is derived.
%
Second, learning dynamics of BN and the regularization show that training converged with large maximum and effective learning rate.
Third, generalization of BN is explored by using statistical mechanics.
% is reformulated as statistical
Experiments demonstrate that BN in convolutional neural networks share the same traits of regularization as the above analyses.
% Finally, the characteristics of BN are studied by .
% support our analyses.
%
\vspace{-5pt}
\end{abstract}

\section{Introduction}\label{sec:intro}

Batch Normalization (BN) is an indispensable component in many deep neural networks \citep{resnet,densenet}. BN has been widely used in various areas such as machine vision, speech and natural language processing.
Experimental studies \citep{BN} suggested that BN improves convergence and generalization by enabling large learning rate and preventing overfitting when training deep networks.
Understanding BN theoretically is a key question.
% and analytically study is still neglected in the literature

%\textbf{Notations.}

This work investigates regularization of BN as well as its optimization and generalization in a single-layer perceptron, which is a building block of deep models, consisting of a kernel layer, a BN layer, and a nonlinear activation function such as ReLU.
%
%which is  as as shown in Fig.\ref{fig:NN} (a).
%
The computation of BN is written by

\vspace{-15pt}
\begin{equation}\label{eq:BN}
\small
y=g(\hh),~~\hh=\gamma\frac{h-\mu_\mathcal{B}}{\sigma_\mathcal{B}}+\beta~~\mathrm{and}~~h=\w\tran\x.
\end{equation}
\vspace{-15pt}

This work denotes a scalar and a vector by using lowercase letter (\eg $x$) and bold lowercase letter (\eg $\textbf{x}$) respectively.
In Eqn.\eqref{eq:BN},
$y$ is the output of a neuron, $g(\cdot)$ denotes an activation function, $h$ and $\hh$ are hidden values before and after batch normalization, $\w$ and $\x$ are kernel weight vector and network input respectively.
%Eqn.\eqref{eq:BN} shows that
%
In BN, $\mu_\mathcal{B}$ and $\sigma_\mathcal{B}$ represent the mean and standard deviation of $h$. They are estimated within a batch of samples for each neuron independently.
$\gamma$ is a scale parameter and $\beta$ is a shift parameter.
In what follows, Sec.\ref{sec:overview} overviews assumptions and main results, and Sec.\ref{sec:related} presents relationships with previous work.

\vspace{-2pt}
\subsection{Overview of Results}\label{sec:overview}
\vspace{-2pt}

We overview results in three aspects.

$\bullet$ First, Sec.\ref{sec:view} decomposes BN into population normalization (PN) and gamma decay. To better understand BN, we treat a single-layer perceptron with ReLU activation function as an illustrative case.
Despite the simplicity of this case, it is a building block of deep networks and has been widely adopted in theoretical analyses such
as proper initialization \citep{krogh_generalization_1992,advani_high-dimensional_2017}, dropout \citep{wager_dropout_2013}, weight decay and data augmentation \citep{bos_statistical_1998}.
The results in Sec.\ref{sec:view} can be extended to deep neural networks as presented in Appendix \ref{app:deep-reg}.

Our analyses assume that neurons at the BN layer are independent similar to \citep{WN,l2BN,BayeUnEs}, as the mean and the variance of BN are estimated individually for each neuron of each layer.
The form of regularization in this study does not rely on Gaussian assumption on the network input and the weight vector, meaning our assumption is milder than those in \citep{WNdynamic,LN,WN}.

%shows that
%BN could be represented by population normalization (PN) with adaptive gamma decay, where
%the stochastic behaviors of
%$\mu_\mathcal{B}$ and $\sigma_\mathcal{B}$ both impose regularization on $\gamma$ termed gamma decay (Fig.\ref{}). But

Sec.\ref{sec:view} tells us that BN has an explicit regularization form, gamma decay, where $\mu_\mathcal{B}$ and $\sigma_\mathcal{B}$
%These statistics
have different impacts:
%whose regularization strength is determined adaptively from data.
%
(1) $\mu_\mathcal{B}$ {discourages} reliance on a single neuron and {encourages} different neurons to have equal magnitude, in the sense that corrupting individual neuron does not harm generalization. This phenomenon was also found empirically in a recent work \citep{single}, but has not been established analytically.
%In other words, $\mu_\mathcal{B}$ regularizes forward propagation of a neural network.
%
(2) $\sigma_\mathcal{B}$ reduces kurtosis of the input distribution as well as {correlations} between neurons.
% $depended on the Fisher Information Matrix (FIM) of $\gamma$.
%, meaning that BN would punish large gradient norm of $\gamma$.
%This is not found by previous empirical studies.
(3) The regularization strengths of these statistics are {inversely proportional} to the batch size $M$, indicating that BN with large batch would decrease generalization. %(Fig.\ref{fig:cifar_finetune} and \ref{fig:cifar_gamma}).
(4)
%$\mu_\mathcal{B}$ and $\sigma_\mathcal{B}$ jointly regularize training.
Removing either one of $\mu_\mathcal{B}$ and $\sigma_\mathcal{B}$ could imped convergence and generalization.
% (Fig.\ref{fig:cifar10_loss}).

%
$\bullet$ Second, by using ordinary differential equations (ODEs),
Sec.\ref{sec:dynamic} shows that gamma decay enables the network trained with BN to converge with {large maximum learning rate and effective learning rate},
%leading to faster training speed
compared to the network trained without BN or trained with weight normalization (WN) \mbox{\citep{WN}} that is a counterpart of BN.
The maximum learning rate (LR) represents the largest LR value that allows training to converge to a fixed point without diverging, while effective LR represents the actual LR in training. Larger maximum and effective LRs imply faster convergence rate.
%The magnitude of the learning rates depend on the regularization strength in BN.
%

$\bullet$ Third, Sec.\ref{sec:gen} compares generalization errors of BN, WN, and vanilla SGD by using statistical mechanics. The ``large-scale'' regime is of interest, where number of samples $P$ and number of neurons $N$ are both large
%($P,N\rightarrow\infty$),
but their ratio $P/N$ is finite.
In this regime, the generalization errors are quantified both analytically and empirically.
%
%BN reduces overtraining especially when $\alpha=1$ ($P$ and $N$ are comparable), outperforming the others.

Numerical results in Sec.\ref{sec:exp} show that BN in CNNs has the same traits of regularization as disclosed above.
% three aspects.
%support our analyses by training a convolutional neural network with multiple hidden layers on real data.

\vspace{-3pt}
\subsection{Related Work}\label{sec:related}
\vspace{-3pt}

\textbf{Neural Network Analysis}.
Many studies conducted theoretical analyses of neural networks \citep{optimal-perceptron,on-line-committe,Dynamics-perceptron,Geometry,Electron-proton,Globally,Expressive,landscape,Yuandong}.
For example, for a multilayer network with linear activation function, \mbox{\cite{diffNN}} explored its SGD dynamics and \cite{wopoor} showed that every local minimum is global.
\cite{Yuandong} studied the critical points and convergence behaviors of a 2-layered network with ReLU units.
%On the other hand, very few for nonlinear networks.
\cite{Electron-proton} investigated a teacher-student model when the activation function is harmonic.
In \citep{on-line-committe}, the learning dynamics of a committee machine were discussed when the activation function is error function $\mathrm{erf}(x)$.
Unlike previous work, this work analyzes regularization emerged in BN and its impact to both learning and generalization, which are still unseen in the literature.

\textbf{Normalization}.
Many normalization methods have been proposed recently.
For example,
BN \citep{BN} was introduced to stabilize the distribution of input data of each hidden layer.
Weight normalization (WN) \citep{WN} decouples the lengths of the network parameter vectors from their directions, by normalizing the parameter vectors to unit length.
The dynamic of WN was studied by using a single-layer network \citep{WNdynamic}.
\citet{Disharmony} diagnosed the compatibility of BN and dropout \citep{dropout} by reducing the variance shift produced by them.

Moreover, \citet{l2BN} showed that weight decay has no regularization effect when using together with BN or WN.
\citet{LN} demonstrated when BN or WN is employed, back-propagating gradients through a hidden layer is scale-invariant with respect to the network parameters. \citet{santurkar_how_2018} gave another perspective of the role of BN during training instead of reducing the covariant shift. They argued that BN results in a smoother optimization landscape and the Lipschitzness is strengthened in networks trained with BN.
However, both analytical and empirical results of regularization in BN are still desirable.
Our study explores regularization, optimization, and generalization of BN in the scenario of online learning.

\textbf{Regularization}.
\citet{BN} conjectured that BN {implicitly} regularizes training to prevent overfitting. \citet{rethinking-generalization} categorized BN as an implicit regularizer from experimental results.
\citet{szegedy_rethinking_2015} also conjectured that in the Inception network, BN behaves similar to dropout to improve the generalization ability. \citet{gitman_comparison_2017} experimentally compared BN and WN, and also confirmed the better generalization of BN.
In the literature there are also {implicit regularization} schemes other than BN.
For instance, random noise in the input layer for data augmentation has long been discovered equivalent to a weight decay method, in the sense that the inverse of the signal-to-noise ratio acts as the decay factor \citep{krogh_generalization_1992,rifai_adding_2011}.
Dropout \citep{dropout} was also
proved able to regularize training by using the generalized linear model \citep{wager_dropout_2013}.
\vspace{-5pt}
\section{A Probabilistic Interpretation of BN}\label{sec:view}
\vspace{-5pt}
%\vspace{5pt}

The notations in this work are summarized in Appendix Table \ref{tab:notation} for reference.
%\textbf{Overview.}

Training the above single-layer perceptron with BN in Eqn.\eqref{eq:BN} typically involves minimizing a negative log likelihood function with respect to a set of network parameters ${\theta}=\{\w,\gamma,\beta\}$.
%, plus the weight decay.
%
Then the loss function is defined by

\vspace{-15pt}
\begin{equation}\label{eq:loss}
\small
\frac{1}{P}\sum_{j=1}^P\ell(\hat{h}^j)=-\frac{1}{P}\sum_{j=1}^P\log p(y^j|\hat{h}^j;{\theta})+\zeta\|{\theta}\|_2^2,
\end{equation}
\vspace{-15pt}

where $p(y^j|\hat{h}^j;{\theta})$ represents the likelihood function of the network and $P$ is number of training samples.
As Gaussian distribution is often employed as prior distribution for the network parameters, we have a regularization term $\zeta\|{\theta}\|_2^2$ known as weight decay \citep{AlexNet} that is a popular technique in deep learning, where $\zeta$ is a coefficient.
% of the regularion.

To derive regularization of BN, we treat $\mu_\mathcal{B}$ and $\sigma_\mathcal{B}$ as random variables.
%, which are estimated by using a batch of samples.
%
Since one sample $\x$ is seen many times in the entire training course, and at each time $\x$ is presented with the other samples in a batch that is drawn randomly,
$\mu_\mathcal{B}$ and $\sigma_\mathcal{B}$ can be treated as injected random noise for $\x$.

\textbf{Prior of $\boldsymbol{\mu_{\mathcal{B}}},\boldsymbol{\sigma_{\mathcal{B}}}$.} By following \citep{BayeUnEs}, we find that BN also induces Gaussian priors for $\mu_{\mathcal{B}}$ and $\sigma_{\mathcal{B}}$. We have $\mu_{\mathcal{B}}\sim\mathcal{N}(\mu_{\mathcal{P}},\frac{\sigma_{P}^{2}}{M})$ and $
\sigma_{\mathcal{B}}\sim\mathcal{N}(\sigma_{P},\frac{\rho+2}{4M})$,
%, resulting in $l2$ regularization over the scale parameter $\gamma$.We have
where $M$ is batch size, $\mu_{\mathcal{P}}$ and $\sigma_{\mathcal{P}}$ are population mean and standard deviation respectively, and $\rho$ is kurtosis that measures the peakedness of the distribution of $h$.
%We have $\rho=\frac{C-\sigma_{P}^{4}}{\sigma_{P}^{4}}-2$ and $C=\mathbb{E}[(h-\mu_{P})^{4}]$.
% and is defined as:
%A distribution with a positive kurtosis is more peaked than a Gaussian distribution with a heavier tail, while a distribution with a negative kurtosis looks like a loaf of bread.
These priors tell us that $\mu_{\mathcal{B}}$ and $\sigma_{\mathcal{B}}$ would produce Gaussian noise in training. There is a tradeoff regarding this noise. For example, when $M$ is small, training could diverge because the noise is large. This is supported by experiment of BN \citep{GN} where training diverges when $M=2$ in ImageNet \citep{imagenet12}.
% size of $M=2$.
%However, the noise would not harm training too much
When $M$ is large, the noise is small
because $\mu_{\mathcal{B}}$ and $\sigma_{\mathcal{B}}$ get close to $\mu_{\mathcal{P}}$ and $\sigma_{\mathcal{P}}$.
% when $M$ is moderate.
%
% when $M$ is small.
It is known that $M>30$ would provide a moderate noise, as the sample statistics converges in probability to the population statistics by the weak Law of Large Numbers. This is also supported by experiment \citep{BN} where BN with $M=32$ already works well in ImageNet.
%
%Second, $\sigma_\mathrm{B}$ may induce kurtosis noise produced by the batch estimation.
%
\vspace{-2pt}
\subsection{A Regularization Form}
\vspace{-2pt}

The loss function in Eqn.(\ref{eq:loss}) can be written as an expected loss by integrating over the priors of $\mu_\mathcal{B}$ and $\sigma_\mathcal{B}$, that is,
$\frac{1}{P}\sum_{j=1}^P\mathbb{E}_{\mu_\mathcal{B},\sigma_\mathcal{B}}[\ell(\hat{h}^j)]$
where $\mathbb{E}[\cdot]$ denotes expectation.
%
%By analyzing this expected loss function,
%
We show that $\mu_\mathcal{B}$ and $\sigma_\mathcal{B}$ impose regularization on the scale parameter $\gamma$ by decomposing BN into population normalization (PN) and gamma decay.
% but result in different regularization strengths.
%
To see this, we employ a single-layer perceptron and ReLU activation function as an illustrative example. A more rigorous description is provided in Appendix \ref{app:theorem}.

%We present regularization emerged in BN in theorem \ref{theorem:reg} and then explain its rationale.

\textbf{Regularization of $\boldsymbol{\mu_\mathcal{B}},\boldsymbol{\sigma_\mathcal{B}}$.}
%
%\unskip\pagebreak
Let $\ell(\hat{h})$ be the loss function defined in Eqn.\eqref{eq:loss} and ReLU be the activation function. We have

\vspace{-17pt}
\begin{equation} %\label{eq:theorem1}
\small
%\mathbb{E}_{(x,y)\sim p_{xy}}
\frac{1}{P}\sum_{j=1}^P\mathbb{E}_{\mu_{\mathcal{B}},\sigma_{\mathcal{B}}}\ell(\hat{h}^j)\simeq
%\mathbb{E}_{xy\simp_{xy}}
\underbrace{\frac{1}{P}\sum_{j=1}^P\ell(\bar{h}^j)}_{\mathrm{PN}}~+\underbrace{\zeta(h)\gamma^2}_{\mathrm{gamma~decay}},
~~\mathrm{and}~~\zeta(h)=\underbrace{\frac{\rho+2}{8M}\mathcal{I}(\gamma)}_{\mathrm{from~} \sigma_{\mathcal{B}}}+\underbrace{\frac{1}{2M}\frac{1}{P}\sum_{j=1}^P\sigma(\bar{h}^j)}_{\mathrm{from~} \mu_{\mathcal{B}}},\label{eq:theorem1}
\end{equation}
\vspace{-13pt}

where $\bar{h}^j=\gamma\frac{h^j-\mu_{\mathcal{P}}}{\sigma_{\mathcal{P}}}+\beta$ and $h^j=\w\tran\x^j$ represent the computations of PN.
$\zeta(h)\gamma^2$ represents gamma decay, where $\zeta(h)$ is an adaptive decay factor depended on the hidden value $h$. Moreover, $\rho$ is the kurtosis of distribution of $h$, $\mathcal{I}(\gamma)$ represents an estimation of the Fisher information of $\gamma$ and $\mathcal{I}(\gamma)=\frac{1}{P}\sum_{j=1}^P(\frac{\partial\ell(\hat{h}^j)}{\partial\gamma})^2$, and $\sigma(\cdot)$ is a sigmoid function.

%\vspace{10pt}
%\begin{theorem}[Regularization of $\mu_\mathcal{B},\sigma_\mathcal{B}$]\label{theorem:reg}
%%
%%\unskip\pagebreak
%Let $\ell(\hat{h})$ be the loss function of BN and the activation function be ReLU. Then
%\vspace{-17pt}
%\begin{equation} \label{eq:theorem1}
%\small
%%\mathbb{E}_{(x,y)\sim p_{xy}}
%\frac{1}{P}\sum_{j=1}^P\mathbb{E}_{\mu_{\mathcal{B}},\sigma_{\mathcal{B}}}\ell(\hat{h}^j)\simeq
%%\mathbb{E}_{xy\simp_{xy}}
%\frac{1}{P}\sum_{j=1}^P\ell(\bar{h}^j)+\zeta(h)\gamma^2
%~~~\mathrm{and}~~~\zeta(h)=\underbrace{\frac{\rho+2}{8M}F_\gamma}_{\mathrm{from~} \sigma_{\mathcal{B}}}+\underbrace{\frac{1}{2M}\frac{1}{P}\sum_{j=1}^P\sigma(\bar{h}^j)}_{\mathrm{from~} \mu_{\mathcal{B}}},%\label{eq:reg2}
%\end{equation}
%\vspace{-13pt}
%
%where $\bar{h}^j=\gamma\frac{h^j-\mu_{\mathcal{P}}}{\sigma_{\mathcal{P}}}+\beta$ represents the population normalization (PN) and $h^j=\w\tran\x^j$. $\zeta(h)$ is a data-dependent coefficient of gamma decay, $\rho$ is the kurtosis of distribution of $h$, $F_\gamma$ represents Fisher Information Matrix (FIM) of $\gamma$, and $\sigma(\cdot)$ is a sigmoid function.
%\end{theorem}

%Theorem \ref{theorem:reg} is derived by performing Taylor expansion of the expected loss function at $\bar{h}^j$ %=\gamma\frac{\mathbf{w}^{T}\mathbf{x}^j-\mu_{\mathcal{P}}}{\sigma_{\mathcal{P}}}+\beta$,
%where the high-order terms are carefully reduced.

From Eqn.\eqref{eq:theorem1}, we have several observations that have both theoretical and practical values.

$\bullet$ First,
%it decomposes BN into population normalization (PN) and gamma decay.
%
%The computation of PN is defined by $\bar{h}^j=\gamma\frac{\mathbf{w}^{T}\mathbf{x}^j-\mu_{\mathcal{P}}}{\sigma_{\mathcal{P}}}+\beta$, where
PN replaces the batch statistics $\mu_{\mathcal{B}},\sigma_{\mathcal{B}}$ in BN by the population statistics $\mu_{\mathcal{P}},\sigma_{\mathcal{P}}$.
In gamma decay, computation of $\zeta(h)$ is {data-dependent}, making it differed from weight decay where the coefficient is determined manually.
In fact, Eqn.\eqref{eq:theorem1} recasts the randomness of BN in a deterministic manner, not only enabling us to apply methodologies such as ODEs and statistical mechanics to analyze BN, but also inspiring us to imitate BN's performance by WN without computing batch statistics in empirical study.

$\bullet$ Second, PN is closely connected to WN, which is independent from sample mean and variance. WN \citep{WN} is defined by $\upsilon\frac{{\w}^{T}{\x}}{{||\w||_2}}$ that normalizes the weight vector $\w$ to have unit variance, where $\upsilon$ is a learnable parameter.
Let each diagonal element of the covariance matrix of $\x$ be $a$ and all the off-diagonal elements be zeros. $\bar{h}^j$ in Eqn.\eqref{eq:theorem1} can be rewritten as

\vspace{-15pt}
\begin{equation}\label{eq:h}
\small
\bar{h}^j=\gamma\frac{\mathbf{w}^{T}\mathbf{x}^j-\mu_{\mathcal{P}}}{\sigma_{\mathcal{P}}}+\beta=\upsilon\frac{{\w}^{T}{\x}^j}{{||\w||_2}}+b,
\end{equation}
\vspace{-15pt}

where $\upsilon=\frac{\gamma}{a}$ and $b=-\frac{\gamma\mu_{\mathcal{P}}}{a{||\w||_2}}+\beta$.
%
%This equation is similar to WN with an additional bias $b$.
%
%We disclose that $\bar{h}^j$ has the form of WN when $\x$ follows a Gaussian distribution with a diagonal covariance matrix.
%
Eqn.(\ref{eq:h}) removes the estimations of statistics and eases our analyses of regularization for BN.
%, although the constraint on $\x$ might not be satisfied in practice.
%, it and

$\bullet$ Third, $\mu_\mathcal{B}$ and $\sigma_\mathcal{B}$ produce different strengths in $\zeta(h)$.
As shown in Eqn.(\ref{eq:theorem1}), the strength from $\mu_\mathcal{B}$ depends on the expectation of $\sigma(\bar{h}^j)\in[0,1]$, which represents excitation or inhibition of a neuron,
meaning that {a neuron with larger output may exposure to larger regularization}, encouraging different neurons to have equal magnitude.
%
%Therefore, the regularization by $\mu_\mathcal{B}$ encourages different neurons to have equal magnitude and discourages the reliance on a single neuron.
%
This is consistent with empirical result \citep{single} which prevented reliance on single neuron to improve generalization.
The strength from $\sigma_\mathcal{B}$ works as a complement for $\mu_\mathcal{B}$.
% by reducing correlations between neurons.
%by penalizing $\rho$ and $F_\gamma$.
%, pre reduce neuron coadaptation.
For a single neuron, $\mathcal{I}(\gamma)$ represents the norm of gradient, implying that BN punishes large gradient norm.
For multiple neurons, $\mathcal{I}(\gamma)$ is the Fisher information matrix of $\gamma$, meaning that BN would penalize correlations among neurons.
% similar to dropout \citep{dropout}.
%
Both $\sigma_\mathcal{B}$ and $\mu_\mathcal{B}$ are important, removing either one of them would imped performance.

\textbf{Extensions to Deep Networks.} The above results can be extended to deep networks as shown in Appendix \ref{app:deep-reg} by decomposing the expected loss at a certain hidden layer.
%marginalizing out the means and variances of a .
We also demonstrate the results empirically in Sec.\ref{sec:exp},
%by performing
%
where we observe that CNNs trained with BN share similar traits of regularization as discussed above.
%, although our results are derived in a single-layer perceptron.
%following an independence assumption of hidden neurons \cite{l2BN},
%
%However, in deep models the priors for $\sigma_\mathcal{B}$ and $\mu_\mathcal{B}$ become multivariate Gaussian distributions where relationships between layers may not be neglected.
%
%In this case, we didn't find meaningful analytical form for BN.
%, where the derivation does not have form.
%

\vspace{-2pt}
\section{Optimization with Regularization}\label{sec:dynamic}
\vspace{-2pt}

Now we show that BN converges with large maximum and effective learning rate (LR), where the former one is the largest LR when training converged, while the latter one is the actual LR during training.
With BN, we find that both LRs would be larger than a network trained without BN.
Our result explains why BN enables large learning rates used in practice \citep{BN}.
%
%To our knowledge, the analysis of the maximum and the effective learning rate for BN is presented for the first time.

Our analyses are conducted in three stages. First, we establish dynamical equations of a teacher-student model in the thermodynamic limit and acquire the fixed point.
Second, we investigate the eigenvalues of the corresponding Jacobian matrix at this fixed point.
Finally, we calculate the maximum and the effective LR.

\textbf{Teacher-Student Model}. We first introduce useful techniques from statistical mechanics (SM). With SM, a student network is dedicated to learn relationship between a Gaussian input and an output by using a weight vector $\w$ as parameters. It is useful to characterize behavior of the student by using a teacher network with $\w^\ast$ as a ground-truth parameter vector.
We treat single-layer perceptron as a student, which is optimized by minimizing the euclidian distance between its output and the supervision provided by a teacher without BN.
The student and the teacher have identical activation function.
%

% converging towards the fixed point.

\textbf{Loss Function.} We define a loss function of the above teacher-student model by $\frac{1}{P}\sum_{j=1}^P\ell(\x^j)=\frac{1}{P}\sum_{j=1}^P\big[g({\w^\ast}\tran\x^j)-
g(\sqrt{N}\gamma\frac{\w\tran\x^j
}{\|\w\|_2})\big]^2+\zeta\gamma^2$, where
$g({\w^\ast}\tran\x^j)$ represents supervision from the teacher, while $g(\sqrt{N}\gamma\frac{\w\tran\x^j
}{\|\w\|_2})$ is the output of student trained to mimic the teacher.
This student is defined by following Eqn.(\ref{eq:h}) with $\nu=\sqrt{N}\gamma$ and the bias term is absorbed into $\w$.
The above loss function represents BN by using WN with gamma decay, and it is sufficient to study the learning rates of different approaches.
%As we are interested to investigate gamma decay, weight decay is not included.
%
Let $\theta=\{{\w},\gamma\}$ be a set of parameters updated by SGD, \ie
%The loss function is minimized by SGD where
%
%$\theta$ is updated by
$\theta^{j+1}=\theta^{j}-\eta\frac{\partial\ell(\x^j)}{\partial\theta^j}$ where $\eta$ denotes learning rate.
The update rules for $\w$ and $\gamma$ are

\vspace{-15pt}
%\begin{small}
\begin{equation}
\small
{\w^{j+1}} -\w^{j}={ \eta\delta^j(\frac{\gamma^j\sqrt{N}}{\|\w^j\|_2}\x^j-
\frac{{\swj}\tran\x^j}{\|\w^j\|_2^2}\w^j)}~~~\mathrm{and}~~~ {\gamma^{j+1}} -\gamma^{j}=\eta(\frac{\delta^j\sqrt{N}{\w^j}\tran\x^j}{\|\w^j\|_2}
-\zeta\gamma^j), \label{eq:w}
\end{equation}
%\end{small}
\vspace{-15pt}

where $\sw^j$ denotes a normalized weight vector of the student, that is, $\sw^j=\sqrt{N}\gamma^j\frac{\w^j}{\|\w^j\|_2}$, and $\delta^j=g'(\swj\tran\x^j)[g({\w^\ast}\tran\x^j)-
g(\swj\tran\x^j)]$ represents the gradient\footnote{$g'(x)$ denotes the first derivative of $g(x)$.} for clarity of notations.
% and $\sw=\sqrt{N}\gamma\frac{\w}{\|\w\|_2}$ as defined in Sec.\ref{sec:gen}.
%

\textbf{Order Parameters}. As we are interested in the ``large-scale'' regime where both $N$ and $P$ are large and their ratio $P/N$ is finite, it is difficult to examine a student with parameters in high dimensions directly.
Therefore, we transform the weight vectors to order parameters that fully characterize interactions between the student and the teacher network.
In this case, the parameter vector can be reparameterized by using a vector of three elements including $\gamma$, $R$, and $L$.
In particular,
$\gamma$ measures length of the normalized weight vector $\sw$, that is, ${\sw}\tran{\sw}=N\gamma^2\frac{{\w}\tran{\w}}{\|\w\|_2^2}=N\gamma^2$.
The parameter $R$ measures angle (overlapping ratio) between the weight vectors of student and teacher.
We have $R=\frac{\sw\tran\w^\ast}{\|\sw\|\|\w^\ast\|}=\frac{1}{N\gamma}\sw\tran\w^\ast$, where the norm of the ground-truth vector is $\frac{1}{N}\w^\ast\tran\w^\ast=1$.
%, since $\|\w^\ast\|=\frac{1}{\sqrt{N}}$.
%
Moreover, $L$ represents length of the original weight vector $\w$ and $L^2=\frac{1}{N}\w\tran\w$.
%
%
%With the above definitions, relationship between $R$ and $L$ can be represented by $RL=\frac{1}{N}\w\tran\w^\ast$.

%\vspace{-2pt}
%\subsection{Learning Dynamics of Order Parameters}
%\vspace{-2pt}

\textbf{Learning Dynamics.}
The update equations \eqref{eq:w} can be transformed into a set of differential equations (ODEs) by using the above order parameters. This is achieved by treating the update step $j$ as a continuous time variable $t=\frac{j}{N}$.
They can be turned into differential equations because
the contiguous step $\Delta t=\frac{1}{N}$ approaches zero in the thermodynamic limit when $N\rightarrow\infty$.
We obtain a dynamical system of three order parameters

\vspace{-15pt}
\begin{equation}
\frac{d \gamma}{d t}=\eta\frac{I_{1}}{\gamma}-\eta\zeta \gamma,\label{eq:Q}~~~
\frac{d R}{dt}=\eta\frac{\gamma}{{L}^2}I_{3}-\eta\frac{R}{{L}^2}I_{1}
-\eta^{2}\frac{\gamma^{2}R}{2{L}^{4}}I_{2},~~~\mathrm{and}~~~
\frac{d {L}}{dt}=\eta^{2}\frac{\gamma^{2}}{2{L}^{3}}I_{2},
\end{equation}
\vspace{-13pt}

where $I_1=\mathds{E}_\x[\delta{\sw}\tran\x]$, $I_2=\mathds{E}_\x[\delta^2\x\tran\x]$, and $I_3=\mathds{E}_\x[\delta{\w^\ast}\tran\x]$ are defined to simplify notations.
The derivations of Eqn.\eqref{eq:Q} can be found in {Appendix} \ref{app:dyn}.

\vspace{-2pt}
\subsection{Fixed Point of the Dynamical System}\label{sec:fix-suff-data}
\vspace{-2pt}

\begin{wraptable}{r}{6cm}
\centering
\scriptsize\vspace{-10pt}
\begin{tabular}{p{8pt}<{\centering}|p{37pt}<{\centering}|p{55pt}<{\centering}|p{26pt}<{\centering}}
\hline
  &  $(\gamma_0,R_0,L_0)$  & $\eta_{\max}$ ($R$) & $\eta_\eff$ ($R$)\\
\hline
 BN & $(\gamma_0,1,L_0)$  & $\big(\frac{\partial(\gamma_0 I_3-I_1)}{\gamma_0 \partial R}-\zeta\gamma_0\big)/\frac{\partial I_2}{2\partial R}$ & $\frac{\eta\gamma_0}{L_0^2}$\\
 WN & $(1,1,L_0)$  & $\frac{\partial(I_3-I_1)}{\partial R}/\frac{\partial I_2}{2\partial R}$ & $\frac{\eta}{L^2_0}$\\
 SGD & $(1,1,1)$ & $\frac{\partial(I_3-I_1)}{\partial R}/\frac{\partial I_2}{2\partial R}$ & $\eta$\\
 \hline
\end{tabular}\vspace{-3pt}
\caption{{\small Comparisons of fixed points, $\eta_{\max}$ for $R$, and $\eta_\eff$ for $R$. A fixed point is denoted as $(\gamma_0,R_0,L_0)$.
%Three methods are compared including BN, WN, and the vanilla SGD without both of them.
}} \label{tab:lr}
\end{wraptable}
To find the fixed points of (\ref{eq:Q}),
%investigate the learning rates of BN, we derive the fixed point of (\ref{eq:Q})
%. As we are interested in the learning rate when training converged, the above system can be linearized at small $\eta$ by neglecting the terms proportional to $\eta^2$,
%
%
%because the learning rate is always decayed to a small magnitude in order to converge to a fixed point.
%
%After that, the solution for Eqn.(\ref{eq:Q}-\ref{eq:L}) follows from
we set $d\gamma/dt=dR/dt=d{L}/dt=0$.
The fixed points of BN, WN, and vanilla SGD (without BN and WN) are given in Table \ref{tab:lr}.
%, denoted as $(\gamma_0,R_0,L_0)$,
%
In the thermodynamic limit, the optima denoted as $(\gamma_0,R_0,L_0)$ would be $(\gamma_0,R_0,L_0)=(1,1,1)$.
Our main interest is the overlapping ratio $R_0$ between the student and the teacher, because it optimizes the direction of the weight vector regardless of its length.
We see that $R_0$ for all three approaches attain optimum `1'.
Intuitively, in BN and WN, this optimal solution does not depend on the value of $L_0$ because their weight vectors are normalized.
In other words, WN and BN are easier to optimize than vanilla SGD, unlike SGD where both $R_0$ and $L_0$ have to be optimized to `1'.
Furthermore, $\gamma_0$ in BN depends on the activation function. For ReLU, we have $\gamma_0^\bn=\frac{1}{2\zeta+1}$ (see Proposition \ref{prop:fixp} in Appendix \ref{app:dyn}), meaning that norm of the normalized weight vector relies on the decay factor $\zeta$.
In WN, we have $\gamma_0^\wn=1$ as WN has no regularization on $\gamma$.
%
%More discussions are provided in the {supplementary material}.
%Note that $\gamma_0^\bn$ in BN would lead to larger learning rates than WN as discussed below.

\vspace{-2pt}
\subsection{Maximum and Effective Learning Rates}
\vspace{-2pt}

With the above fixed points, we derive the maximum and the effective LR.
% when training converged.
%
Specifically, we analyze eigenvalues and eigenvectors of the Jacobian matrix corresponding to Eqn.\eqref{eq:Q}.
%, determining the asymptotic convergence behavior towards the fixed points.
%
We are interested in the LR to approach $R_0$.
%, since it optimizes the update direction of the weight vector.
%By inspecting the Jacobian,
We find that this optimum value only depends on its corresponding eigenvalue denoted as $\lambda_R$. We have
%which has the following form in general
$\lambda_R=\frac{\partial I_2}{\partial R}\frac{\eta{\gamma_0}}{2L_0^2}
(\eta_{\max}-\eta_{\eff})$,
where $\eta_{\max}$ and $\eta_{\eff}$ represent the maximum and effective LR (proposition \ref{prop:eigen} in Appendix \ref{app:dyn}), which are given in Table \ref{tab:lr}.
We demonstrate that $\lambda_R<0$ if and only if $\eta_{\max}>\eta_{\eff}$, such that the fixed point $R_0$ is stable for all approaches (proposition \ref{prop:constraint} in Appendix \ref{app:lr}).
Moreover, it is also able to show that $\eta_{\max}$ of BN ($\eta_{\max}^\bn$) is larger than WN and SGD, enabling $R$ to converge with a larger learning rate.
%
%This result is in line with the empirical findings in \cite{BN}, where deep networks with BN can be trained by using large learning rates.
%
For ReLU as an example, we find that
%$\eta_{\max}^\bn$ could be larger than $\eta_{\max}^\wn$ and $\eta_{\max}^\ord$. We have
$\eta_{\max}^\bn\geq\eta_{\max}^{\{\wn,\ord\}}+2\zeta$ (proposition \ref{prop:maxeta} in Appendix \ref{app:maxlr}).
The larger maximum LRs enables the network to be trained more stably and has the potential to be combined with other stabilization techniques \citep{Robust} during optimization.
The effective LRs shown in Table \ref{tab:lr} are consistent with previous work \citep{l2BN}.

\vspace{-5pt}
\section{Generalization Analysis}\label{sec:gen}
\vspace{-5pt}

Here we investigate generalization of BN by using a teacher-student model that
%with both identity and ReLU as activation functions of the student. This model
minimizes a loss function $\frac{1}{P}\sum_{j=1}^{P}((y^\ast)^{j}-y^{j})^{2}$, where ${y^\ast}$ represents the teacher's output and $y$ is the student's output.
We compare BN with WN+gamma decay and vanilla SGD.
All of them share the same teacher network whose output is a noise-corrupted linear function $y^\ast={\w^{\ast}}\tran\x+s$,
where $\x$ is drawn from $\mathcal{N}(0,\frac{1}{N})$ and $s$ is an unobserved Gaussian noise.
We are interested to see how the above methods resist this noise by using student networks with both identity (linear) and ReLU activation functions.

For \textbf{vanilla SGD}, the student is computed by $y=g(\w\tran\x)$ with $g(\cdot)$ being either identity or ReLU, and $\w$ being the weight vector to optimize, where $\w$ has the same dimension as $\w^\ast$.
The loss function of vanilla SGD is $\ell^\ord=\frac{1}{P}\sum_{j=1}^{P}\big(y^\ast-g(\w\tran\x^j)\big)^{2}$.
%
%If the student activation is linear, the vanilla SGD solution asymptotically approaches the Moore\textendash Penrose pseudo inverse solution $\w=\left(\x\tran\x\right)^{+}\x\tran{\y}^\ast$. If the student activation is Relu, the solution is much more
%
For \textbf{BN},
%The teacher of BN is computed the same as the vanilla SGD.
%
the student is defined as $y=\gamma\frac{\w\tran\x-\mu_{\mathcal{B}}}{\sigma_{\mathcal{B}}}+\beta$.
%, where $\mu_{\mathcal{B}}$ and $\sigma_{\mathcal{B}}$ are estimated in a batch of $M$ samples.
%
As our main interest is the weight vector, we freeze the bias by setting $\beta=0$.
Therefore, the batch average term $\mu_\mathcal{B}$ is also unnecessary to avoid additional parameters, and the loss function is written as $\ell^\bn=\frac{1}{P}\sum_{j=1}^{P}\big((y^\ast)^j-\gamma\w\tran\x^j/\sigma_{\mathcal{B}}\big)^{2}$.
For \textbf{WN+gamma decay},
%Here the teacher is also defined identically as the vanilla SGD.
the student is computed similar to Eqn.\eqref{eq:h} by using $y=\sqrt{N}\gamma\frac{\w\tran\x}{\|\w\|_2}$.
%, where BN is replaced by WN.
% since the covariance matrix of $\x$ is a diagonal matrix.
%
%Similar to BN above, we freeze $\beta=0$.
%
Then the loss function is defined by $\ell^\wn=\frac{1}{P}\sum_{j=1}^P\big((y^\ast)^j-\sqrt{N}\gamma\frac{\w\tran\x^j
}{\|\w\|_2}\big)^2+\zeta\|\gamma\|^2_2$.
%
%The decay rate $\zeta$ could be derived by applying theorem \ref{theorem:reg}.
%
%In the above teacher-student model with identity unit, expression of $\zeta$ becomes $\zeta=\frac{1}{2M}$ (Appendix \ref{app:theorem-id}).
%With the identity activation and assuming $\rho=0$ (Gaussian distribution), the expression of $\zeta$.
With the above definitions, the three approaches are studied under the same teacher-student framework, where
their generalization errors can be strictly compared with the other factors ruled out.
%
%We present both analytical and numerical solutions below.
%

\vspace{-5pt}
\subsection{Generalization Errors}
\vspace{-2pt}

\begin{wrapfigure}{r}{0.35\textwidth}
\vspace{-40pt}
  \begin{center}
    \includegraphics[width=0.35\textwidth]{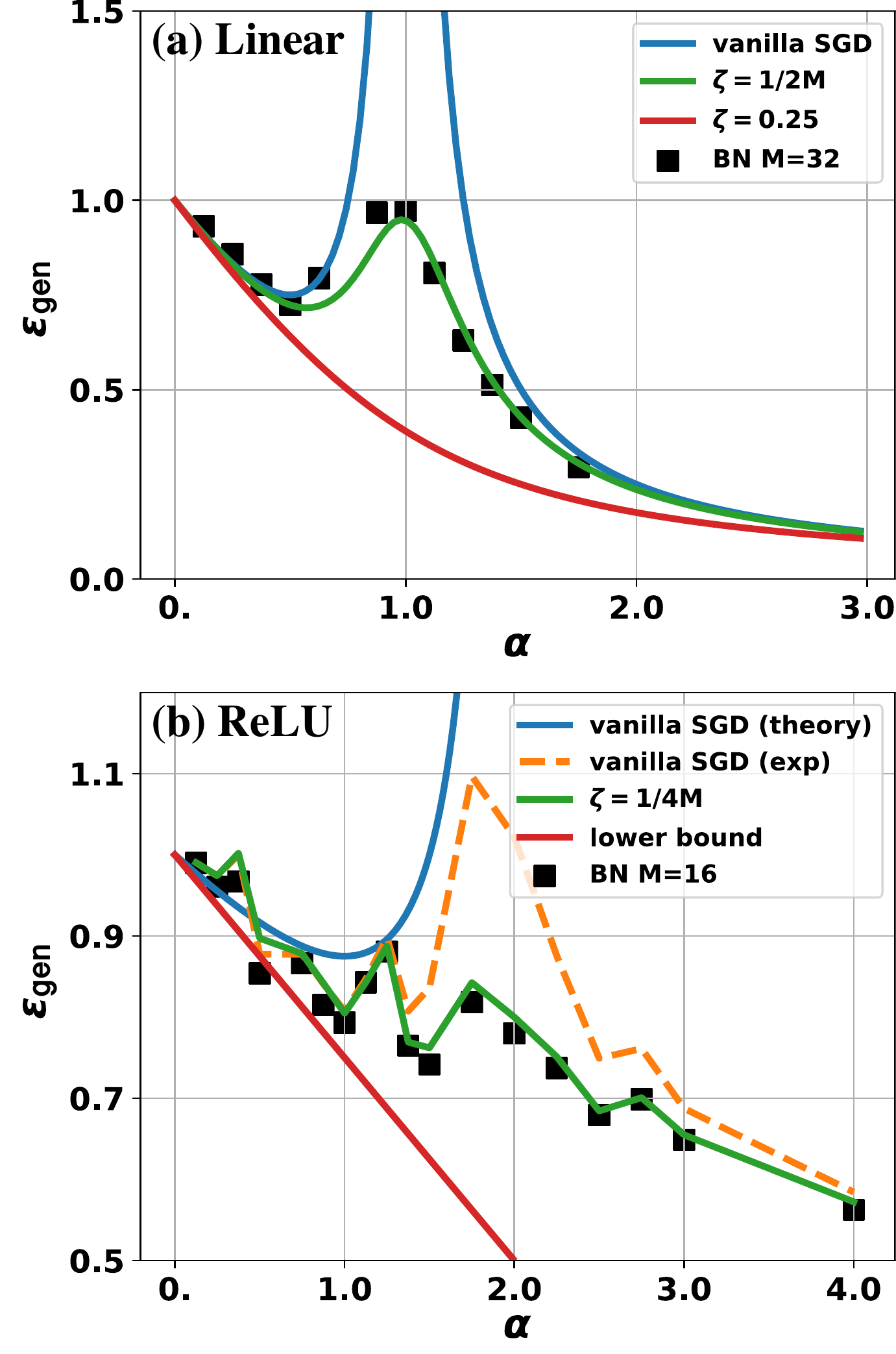}
  \end{center}
  \vspace{-15pt}
  \caption{{\small \textbf{(a)} shows generalization error \vs effective load $\alpha$ using a linear student (identity units). `WN+gamma decay' has two curves $\zeta=\frac{1}{2M}$ and $\zeta=0.25$. BN is trained with $M=32$. \textbf{(b)} shows generalization error \vs effective load $\alpha$ using a ReLU student. `WN+gamma decay' has $\zeta=\frac{1}{4M}$ and is compared to BN with batch size $M=32$. The theoretical curve for vanilla SGD is also shown in blue. The red line is the generalization error of vanilla SGD with no noise in the teacher and thus serves as a lower bound. \vspace{-20pt}
  }}\label{fig:st_loss}\vspace{-10pt}
\end{wrapfigure}
We provide closed-form solutions of the generalization errors (see Appendix \ref{sub:gen}) for vanilla SGD with both linear and ReLU student networks. The theoretical solution of WN+gamma decay can also be solved for the linear student, but still remains difficult for ReLU student whose numerical verification is provided instead.
% and thus a numerical verification is provided for ReLU student.
%
Both vanilla SGD and WN+gamma decay are compared with numerical solutions of BN.

\textbf{vanilla SGD.}
In an identity (linear) student, the solution of generalization error depends on the rank of correlation matrix $\mathbf{\Sigma}=\x\tran\x$. Here we define an effective load $\alpha=P/N$ that is the ratio between number of samples $P$ and number of input neurons $N$ (number of learnable parameters).

The generalization error of the identity student is denoted as $\epsilon_{\mathrm{id}}^\ord$, which can be acquired by using the distribution of eigenvalues of $\mathbf{\Sigma}$ following \citep{advani_high-dimensional_2017}. If $\alpha<1$, $\epsilon_{\mathrm{id}}^\ord=1-\alpha+{\alpha S}{/(1-\alpha)}$. Otherwise, $\epsilon_{\mathrm{id}}^\ord={S}{/(\alpha-1)}$ where
$S$ is the variance of the injected noise to the teacher network.
The values of $\epsilon_{\mathrm{id}}^\ord$ with respect to $\alpha$ are plotted in blue curve of Fig.\ref{fig:st_loss}(a). It first decreases but then increases as $\alpha$ increases from 0 to 1. $\epsilon_{\mathrm{id}}^\ord$ diverges at $\alpha=1$. And it would decrease again when $\alpha>1$.

In a ReLU student,
%the correlation matrix $\mathbf\Sigma}$ also depends on weights $\w$ and thus add difficulty in
the nonlinear activation yields difficulty to derive the theoretical solution. Here we utilize the statistical mechanics and calculate that $\epsilon_{\mathrm{relu}}^\ord = 1-\alpha/4+\frac{\alpha S}{2(2-\alpha)}$ and $\alpha<2$ (see Appendix\ref{sub:equi_order}).
When comparing to the lower bound (trained without noisy supervision) shown as the red curve in Fig.\ref{fig:st_loss}(b),
we see that $\epsilon_{\mathrm{relu}}^\ord$ (blue curve) diverges at $\alpha=2$. This is because the student overfits the noise in the teacher's output.
The curve of numerical solution is also plotted in dashed line in Fig.\ref{fig:st_loss}(b) and it captures the diverging trend well.
It should be noted that obtaining the theoretical curve empirically requires an infinitely long time of training and an infinitely small learning rate. This unreachable limit explains the discrepancies between the theoretical and the numerical solution.
%in absolute values.

\textbf{WN+gamma decay.} In a linear student, the gamma decay term turns the correlation matrix to $\mathbf{\Sigma}=\left(\mathbf{x}\tran\mathbf{x}+\zeta\mathbf{I}\right)$, which is positive definite.
Following statistical mechanics \citep{krogh_generalization_1992}, the generalization error is
$
\epsilon_{\mathrm{id}}^\wn=\delta^{2}\frac{\partial\left(\zeta G\right)}{\partial\zeta}-\zeta^{2}\frac{\partial G}{\partial\zeta}$ {where} $G={1-\alpha-\zeta+{\big(\zeta+(1+\sqrt{\alpha})^{2}\big)^{\frac{1}{2}}
\big(\zeta+(1-\sqrt{\alpha})^{2}\big)^{\frac{1}{2}}}}\big/{2\zeta}.
$
We see that $\epsilon_{\mathrm{id}}^\wn$ can be computed quantitatively given the values of $\zeta$ and $\alpha$.
%In this regard, $\epsilon_{\mathrm{gen}}$ of WN with gamma decay can be computed quantitatively given values of $\zeta$ and $\epsilon$.
%
Let the variance of noise injected to the teacher be $0.25$.
%The generalization error curves with respect to $\alpha$ are displayed in Fig.\ref{fig:st_loss1}.
%
Fig.\ref{fig:st_loss}(a) shows that no other curves could outperform the red curve when $\zeta=0.25$, a value equal to the noise magnitude.
The $\zeta$ smaller than $0.25$ (green curve $\zeta=\frac{1}{2M}$ and $M=32$) would exhibit overtraining around $\alpha=1$, but they still perform significantly better than vanilla SGD.

For the ReLU student in Fig,\ref{fig:st_loss}(b), a direct solution of the generalization error $\epsilon_{\mathrm{relu}}^\wn$ remains an open problem. Therefore, the numerical results of `WN+gamma decay' (green curve) are run at each $\alpha$ value.
% and are comparable to the results of BN (black square).
%According to Eqn.\eqref{eq:theorem1}, the equivalent decay factor of gamma decay would be $\frac{1}{4M}$ where $M=32$.
% is the corresponding batch size in BN.
%It can be seen that `WN+gamma decay' when $\frac{1}{4M}$ is comparable to BN (black squares) and
It effectively reduces over-fitting compared to vanilla SGD.

\textbf{Numerical Solutions of BN.}
In the linear student, we employ SGD with $M=32$ to find solutions of $\w$ for BN. %The generalization error is evaluated as difference between the validation and the training loss (\ie `validation loss'$-$`training loss').
% with respect to different values of $\alpha$.
The number of input neurons is 4096 and the number of training samples can be varied to change $\alpha$. The results are marked as black squares in Fig.\ref{fig:st_loss}(a).
After applying the analyses for linear student (Appendix \ref{app:theorem-id}), BN is equivalent to `WN+gamma decay' when $\zeta=\frac{1}{2M}$ (green curve).
% has been derived above for.
% such a linear regression problem with no bias
%term, that is, $\zeta=\frac{1}{2M}$.
It is seen that BN gets in line with the curve ‘$\zeta=1/2M$’ ($M=32$) and thus quantitatively validates our derivations.

In the ReLU student, the setting is mostly the same as the linear case, except that we employ a smaller batch size $M=16$. The results are shown as black squares in Fig.\ref{fig:st_loss}(b). For ReLU units, the equivalent $\zeta$ of gamma decay is $\zeta=\frac{1}{4M}$. If one compares the generalization error of BN with `WN+gamma decay' (green curve), a clear correspondence is found, which also validates the derivations for the ReLU activation function.

\vspace{-5pt}
\section{Experiments in CNNs}\label{sec:exp}
\vspace{-5pt}

This section shows that BN in CNNs follows similar traits of regularization as the above analyses.

%%%%\begin{wrapfigure}{r}{0.35\textwidth}
%%%%\vspace{-19pt}
%%%%  \begin{center}
%%%%    \includegraphics[width=0.35\textwidth]{cifar-PN.pdf}
%%%%  \end{center}
%%%%  \vspace{-15pt}
%%%%  \caption{{\small Training and evaluation loss in (a) and validation accuracy in (b).\vspace{-20pt}
%%%%  }}\label{fig:PN_CIFAR2}\vspace{-15pt}
%%%%\end{wrapfigure}

%\begin{figure}{t}
%\vspace{-19pt}
%  \begin{center}
%    \includegraphics[width=1\textwidth]{FIGURE2.pdf}
%  \end{center}
%  \vspace{-15pt}
%  \caption{\small{ (a) & (b) compares the loss (both training and evaluation) and validation accuracy between BN and PN on CIFAR-10 using a Resnet-18 network; (c) & (d) compares the training and validation loss curve with WN + mean-only BN and WN + variance-only BN; (e) & (f) validates the regularization effect of BN on both $\gamma^2$ and the validation loss with different batch sizes; (g) & (h) shows the loss and top-1 validation accuracy of Resnet-18 with additional regularizations (dropout) on large-batch training of BN and WN.}
%  }\label{fig:FIGURE2}\vspace{-15pt}
%\end{figure}
%
To compare different methods, the CNN architectures are fixed while only the normalization layers are changed. We adopt CIFAR10 \citep{cifar} that contains 60k images of 10 categories (50k images for training and 10k images for test).
All models are trained by using SGD with momentum, while the initial learning rates are scaled proportionally \citep{goyal_accurate_2017} when different batch sizes are presented.
%
%In order to study regularization of BN, we discard any other trick such as weight decay and data augmentation.
%
More empirical setting can be found in Appendix \ref{app:exp}.

%The latter one has identical 1.2 million data and 1k categories as the original ImageNet, but each image is scaled to 32$\times$32.
%
%For all experiments in CIFAR10, we adopt a 6-layered CNN similar to \cite{WN}
%

\textbf{Evaluation of PN+Gamma Decay}.
This work shows that BN can be decomposed into PN and gamma decay. We empirically compare `PN+gamma decay' with BN by using
%We compare BN with `PN+gamma decay' using
ResNet18 \citep{resnet}. For `PN+gamma decay', the population statistics of PN and the decay factor of gamma decay are estimated by using sufficient amount of training samples.
For BN, BN trained with a normal batch size $M=128$ is treated as baseline as shown in Fig.\ref{fig:FIGURE2}(a\&b).
% in CIFAR10 and $M=32$ in ImageNet-ds.
% of CIFAR and ImageNet-ds respectively.
%
We see that when batch size increases, BN would imped both loss and accuracy.
For example, when increasing $M$ to $1024$, performance decreases because the regularization from the batch statistics reduces in large batch, resulting in overtraining (see the gap between train and validation loss in (a) when $M=1024$).

In comparison, we train PN by using 10k training samples to estimate the population statistics. Note that this further reduces regularization. We see that the release of regularization can be complemented by gamma decay, making PN outperformed BN. This empirical result verifies our derivation of regularization for BN.
% compared to BN with large batch.
%
Similar trend can be observed by experiment in a down-sampled version of ImageNet (see Appendix \ref{app:imagenet}).
We would like to point out that `PN+gamma decay' is of interest in theoretical analyses, but it is computation-demanding when applied in practice because evaluating $\mu_\mathcal{P}$, $\sigma_\mathcal{P}$ and $\zeta(h)$ may require sufficiently large number of samples.

\textbf{Comparisons of Regularization.} We study the regulation strengths of vanilla SGD, BN, WN, WN+mean-only BN, and WN+variance-only BN.
At first, the strength of regularization terms from both $\mu_\mathcal{B}$ and $\sigma_\mathcal{B}$ are compared by using a simpler network with 4 convolutional and 2 fully connected layers as used in \citep{WN}.
%
%Despite its lower baseline, the stacked CNN structure coincides with the framework of the regularization terms in the current study.
Fig.\ref{fig:FIGURE2}(c\&d) compares their training and validation losses.
We see that the generalization error of BN is much lower than WN and vanilla SGD.
The reason has been disclosed in this work: stochastic behaviors of $\mu_{\mathcal{B}}$ and $\sigma_{\mathcal{B}}$ in BN improves generalization.

To investigate $\mu_\mathcal{B}$ and $\sigma_\mathcal{B}$ individually, we decompose their contributions by running a WN with mean-only BN as well as a WN with variance-only BN, to simulate their respective regularization.
As shown in Fig.\ref{fig:FIGURE2}(c\&d), improvements from the mean-only and the variance-only BN over WN verify our conclusion that noises from $\mu_\mathcal{B}$ and $\sigma_\mathcal{B}$ have different regularization strengths.
Both of them are essential to produce good result.

%\subsection{Ablation Study}

\textbf{Regularization and parameter norm.}
%As shown above, the reduction of generalization errors using BN's statistics verifies Eqn.\eqref{eq:reg1}, where BN is decomposed into WN and regularization on $\gamma$.
%
%
We further demonstrate impact of BN to the norm of parameters.
%
%in BN is less straightforward in deep models, which
%%
%As deep models are known to suffer from multiple local minima, in order
We compare BN with vanilla SGD.
A network is first trained by BN in order to converge to a local minima where the parameters do not change much. At this local minima, the weight vector is frozen and denoted as $\mathbf{w}^{\bn}$.
Then this network is finetuned by using vanilla SGD with a small learning rate $10^{-3}$ and its kernel parameters are initialized by $\mathbf{w}^\ord=\gamma\frac{\mathbf{w}^{\bn}}{{\sigma}}$, where ${\sigma}$ is the moving average of $\sigma_{\mathcal{B}}$.

Fig.\ref{fig:cifar_finetune} in Appendix \ref{app:paramnorm} visualizes the results.
As $\mu_{\mathcal{B}}$ and $\sigma_{\mathcal{B}}$ are removed in vanilla SGD, it is found that the training loss decreases while the validation loss increases, implying that reduction in regularization makes the network converged to a sharper local minimum that generalizes less well.
The magnitudes of kernel parameters $\w^\ord$ at different layers are also observed to increase after freezing BN, due to the release of regularization on these parameters.

%%\begin{wrapfigure}{r}{0.34\textwidth}
%%\vspace{-18pt}
%%  \begin{center}
%%    \includegraphics[width=0.34\textwidth]{batchsize.pdf}
%%  \end{center}
%%  \vspace{-15pt}
%%  \caption{{\small Values of $\gamma^2$ increase along with $M$ as shown in (a), due to the lack of regularization in large batch, making the validation losses increased as well in (b). \vspace{-20pt}
%%  }}\label{fig:cifar_gamma}
%%\end{wrapfigure}
\begin{figure}[t]
%\vspace{-0.7cm}
\begin{center}
\includegraphics[width=\textwidth]{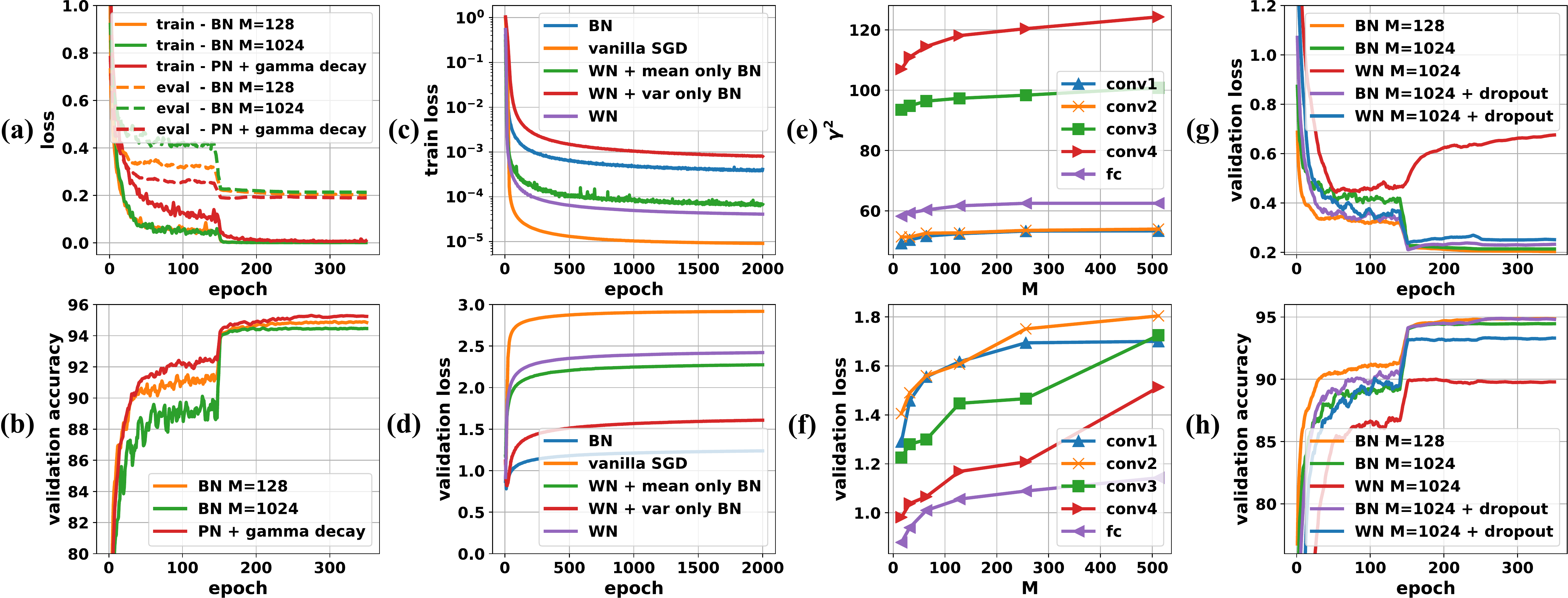}
\end{center}
\vspace{-10pt}
\caption{\small{ (a) \& (b) compare the loss (both training and evaluation) and validation accuracy between BN and PN on CIFAR10 using a ResNet18 network; (c) \& (d) compare the training and validation loss curve with WN + mean-only BN and WN + variance-only BN; (e) \& (f) validate the regularization effect of BN on both $\gamma^2$ and the validation loss with different batch sizes; (g) \& (h) show the loss and top-1 validation accuracy of ResNet18 with additional regularization (dropout) on large-batch training of BN and WN. \vspace{-15pt}}}
\label{fig:FIGURE2}
%\vspace{-0.2cm}
\end{figure}

\textbf{Batch size.} To study BN with different batch sizes,
%we evaluate BN in different layers of a network.
%
we train different networks but only add BN at one layer at a time.
The regularization on the $\gamma$ parameter is compared in Fig.\ref{fig:FIGURE2}(e) when BN is located at different layers. The values of $\gamma^2$ increase along with the batch size $M$ due to the weaker regularization for the larger batches. The increase of $\gamma^2$ also makes all validation losses increased as shown in Fig.\ref{fig:FIGURE2}(f).
%, rise with the batch size, since the regularization reduces at large batch.
% (over 512 in this study).

\textbf{BN and WN trained with dropout}.
%Despite the better generalization of BN with smaller batch sizes, large-batch training is more efficient in real cases. Therefore, improving generalization of BN with large batch is more desiring.
%
As PN and gamma decay requires estimating the population statistics that increases computations,
%We also found that treating the decay factor as a constant hardly improves generalization for large batch.
%
%A closer look at Eqn.(\ref{eq:reg2}) reveals that BN induces an input-data-dependent regularization coefficient $\zeta$, which is difficult to account for in deep networks.
%
we utilize dropout as an alternative to improve regularization of BN.
%Dropout has also been analytically viewed as a regularizer \citep{wager_dropout_2013}.
% and share a similar expression with Eqn.(\ref{eq:reg1}).
%
We add a dropout after each BN layer.
Fig.\ref{fig:FIGURE2}(g\&h) plot the classification results using ResNet18.
%Due to the absence of other regularization methods including data augmentation and weight decay,
The generalization of BN deteriorates significantly when $M$ increases from 128 to 1024. This is observed by the much higher validation loss (Fig.\ref{fig:FIGURE2}(g)) and lower validation accuracy (Fig.\ref{fig:FIGURE2}(h)) when $M=1024$.
If a dropout layer with ratio $0.1$ is added \emph{after} each residual block layer for $M=1024$ in ResNet18, the validation loss is suppressed and accuracy increased by a great margin. This superficially contradicts with the original claim that BN reduces the need for dropout \citep{BN}. As discussed in Appendix \ref{app:bn-wn-dropout}, we find that there are two differences between our study and previous work \citep{BN}.

%Since BN can be treated as WN trained with regularization as shown in this study, combining WN with regularization should be able to match the performance of BN.
% and the performance of WN might also be similarly improved.
%
%As WN outperforms BN in running speed (without calculating statistics) and it suits better in RNNs than BN, an improvement of its generalization is also of great importance.
%
Fig.\ref{fig:FIGURE2}(g\&h) also show that WN can also be regularized by dropout.
%as above.
%
We apply dropout after each WN layer with ratio 0.2 and the dropout is applied at the same layers as that for BN.
%, which in fact also induces regularization on $\gamma$.
We found that the improvement on both validation accuracy and loss is surprising.
%As shown in Fig.\ref{fig:cifar-b},
The accuracy increases from 0.90 to 0.93, even close to the results of BN. Nevertheless, additional regularization on WN still cannot make WN on par with the performance BN. In deep neural networks the distribution after each layer would be far from a Gaussian distribution, in which case WN is not a good substitute for PN.
%A potential substitute of BN would require us for designing better estimations of the distribution to improve the training speed and performance of deep networks.

\vspace{-5pt}
\section{Conclusions}
\vspace{-5pt}

This work investigated an explicit regularization form of BN, which
% By utilizing a single-layer perceptron, BN
was decomposed into PN and gamma decay where the regularization strengths from $\mu_\mathcal{B}$ and $\sigma_\mathcal{B}$ were explored.
%The analytical form of its regularization is presented by studying the stochastic behaviors of $\mu_\mathcal{B}$ and $\sigma_\mathcal{B}$.
Moreover, optimization and generalization of BN with regularization were derived and compared with vanilla SGD, WN, and WN+gamma decay, showing that BN enables training to converge with large maximum and effective learning rate, as well as leads to better generalization.
%whose values depend on the regularization strength coefficient $\zeta$.
%
%For ReLU activation function, we show that $\eta_{\max}^\bn$ could be larger than $\eta_{\max}^\wn$ and $\eta_{\max}^\ord$ by $2\zeta$.
%those without BN.
%We also show that the regularization strength $\zeta$ in BN results in better generalization ability than WN and vanilla SGD.
%
Our analytical results explain many existing empirical phenomena.
Experiments in CNNs showed that BN in deep networks share the same traits of regularization.
% were also observed in training deep convolutional networks.
%
%Furthermore, a combination of dropout and BN might ameliorate BN when batch size goes large.
%
%Our result also encourages us to combine WN and dropout, outperforming BN in some senses without estimating batch statistics.
%
In future work, we are interested in analyzing optimization and generalization of BN in deep networks, which is still an open problem.
Moreover, investigating the other normalizers such as instance normalization (IN) \citep{IN} and layer normalization (LN) \citep{LN} is also important.
Understanding the characteristics of these normalizers should be the first step to analyze some recent best practices such as whitening \citep{GWNN,EigenNet}, switchable normalization \citep{SN,SN2,SSN}, and switchable whitening \citep{SW}.
% that combines BN, IN, and LN in each normalization layer.
%%
%Furthermore, devising an efficient counterpart of gamma decay is desirable in the community and will be investigated in the future, as it may improve generalization of WN that is independent of batch statistics.

%an automatic strategy to regularize and improve large batch training with BN will be developed. An analytical form of optimization and regularization of CNN with BN is desirable and would be .

{\small{
%\bibliography{example_paper}
\bibliographystyle{iclr2019_conference}
\bibliography{icml,egbib}
}}

%\onecolumn
\newpage

\appendix
\section*{Appendices}
\addcontentsline{toc}{section}{Appendices}
\renewcommand{\thesubsection}{\Alph{subsection}}

%Here we outline the proofs of the results presented in the paper.

\subsection{Notations}

\begin{table}[h]
\centering
\small
\vspace{-5pt}
\caption{{Several notations are summarized for reference.
}} \label{tab:notation}
\vspace{2pt}
\begin{tabular}{c|c}
\hline
 $\mu_\mathcal{B},\sigma_\mathcal{B}^2$ & batch mean, batch variance\\
 $\mu_\mathcal{P},\sigma_\mathcal{P}^2$ & population mean, population variance\\
 $\x,y$ & input of a network, output of a network\\
 $y^\ast$ & ground truth of an output\\
 $h,\hat{h}$ & hidden value before and after BN\\
 $\bar{h}$ & hidden value after population normalization\\
 $\gamma,\beta$ & scale parameter, shift parameter\\
 $g(\cdot)$ & activation function\\
 $\w,\w^\ast$ & weight vector, ground truth weight vector\\
 $\sw$ & normalized weight vector\\
 $M,N,P$ &batch size, number of neurons, sample size\\
 $\alpha$ & an effective load value $\alpha=P/N$ \\
 $\zeta$ & regularization strength (coefficient)\\
 $\rho$ & Kurtosis of a distribution\\
 $\delta$ & gradient of the activation function\\
 $\eta_\eff,\eta_{\max}$ & effective, maximum learning rate\\
 $R$ & overlapping ratio (angle) between $\sw$ and $\w^\ast$\\
 $L$ & norm (length) of $\w$\\
 $\lambda_{\max},\lambda_{\min}$ & maximum, minimum eigenvalue\\
 $\epsilon_{\mathrm{gen}}$ & generalization error\\
 \hline
\end{tabular}
\end{table}

\subsection{More Empirical Settings and Results}\label{app:exp}

All experiments in Sec.\ref{sec:exp} are conducted in CIFAR10 by using ResNet18 and a CNN architecture similar to \citep{WN} that is summarized
as
`{conv}(3,32)-conv(3,32)-conv(3,64)-conv(3,64)-pool(2,2)-fc(512)-fc(10)',
where `conv(3,32)' represents a convolution with kernel size 3 and
32 channels, `pool(2,2)' is max-pooling with kernel size 2 and stride
2, and `fc' indicates a full connection.
%
%This CNN is trained on CIFAR-10 \cite{cifar}, which contains 60k natural images of 10 categories, where 50k images are used for training and the remaining images for test.
%
We follow a configuration for training by using SGD with a momentum value of 0.9 and continuously
decaying the learning rate by a factor of $10^{-4}$ each step. For
different batch sizes, the initial learning rate is scaled proportionally
with the batch size to maintain a similar learning dynamics \citep{goyal_accurate_2017}.

%For downsampled ImageNet, we employ ResNet18 and the training and testing strategies are strictly followed those in \citep{resnet}.

\subsubsection{Results in downsampled ImageNet}\label{app:imagenet}

Besides CIFAR10, we also evaluate `PN+gamma decay' by employing a downsampled version of ImageNet \citep{SGDR}, which contains identical 1.2 million data and 1k categories as the original ImageNet, but each image is scaled to 32$\times$32.
%
%For all experiments in CIFAR10, we adopt a 6-layered CNN similar to \cite{WN}.
We train ResNet18 in downsampled ImageNet by following the training protocol used in \citep{resnet}.
In particular, ResNet18 is trained by using SGD with momentum of 0.9 and the initial learning rate is 0.1, which is then decayed by a factor of 10 after 30, 60, and 90 training epochs.

In downsampled ImageNet, we observe similar trends as those presented in CIFAR10. For example, we see that BN would imped both loss and accuracy when batch size increases.
When increasing $M$ to $1024$ as shown in Fig.\ref{fig:PN_imagenet}, both the loss and validation accuracy decrease because the regularization from the random batch statistics reduces in large batch size, resulting in overtraining. This can be seen by the gap between the training and the validation loss.
Nevertheless, we see that the reduction of regularization can be complemented when PN is trained with adaptive gamma decay, which makes PN performed comparably to BN in downsampled ImageNet.

\begin{figure}[h]
\centering
\subfigure[Comparisons of train and validation loss.]{\label{fig:imagenet_a}\includegraphics[width=63mm]{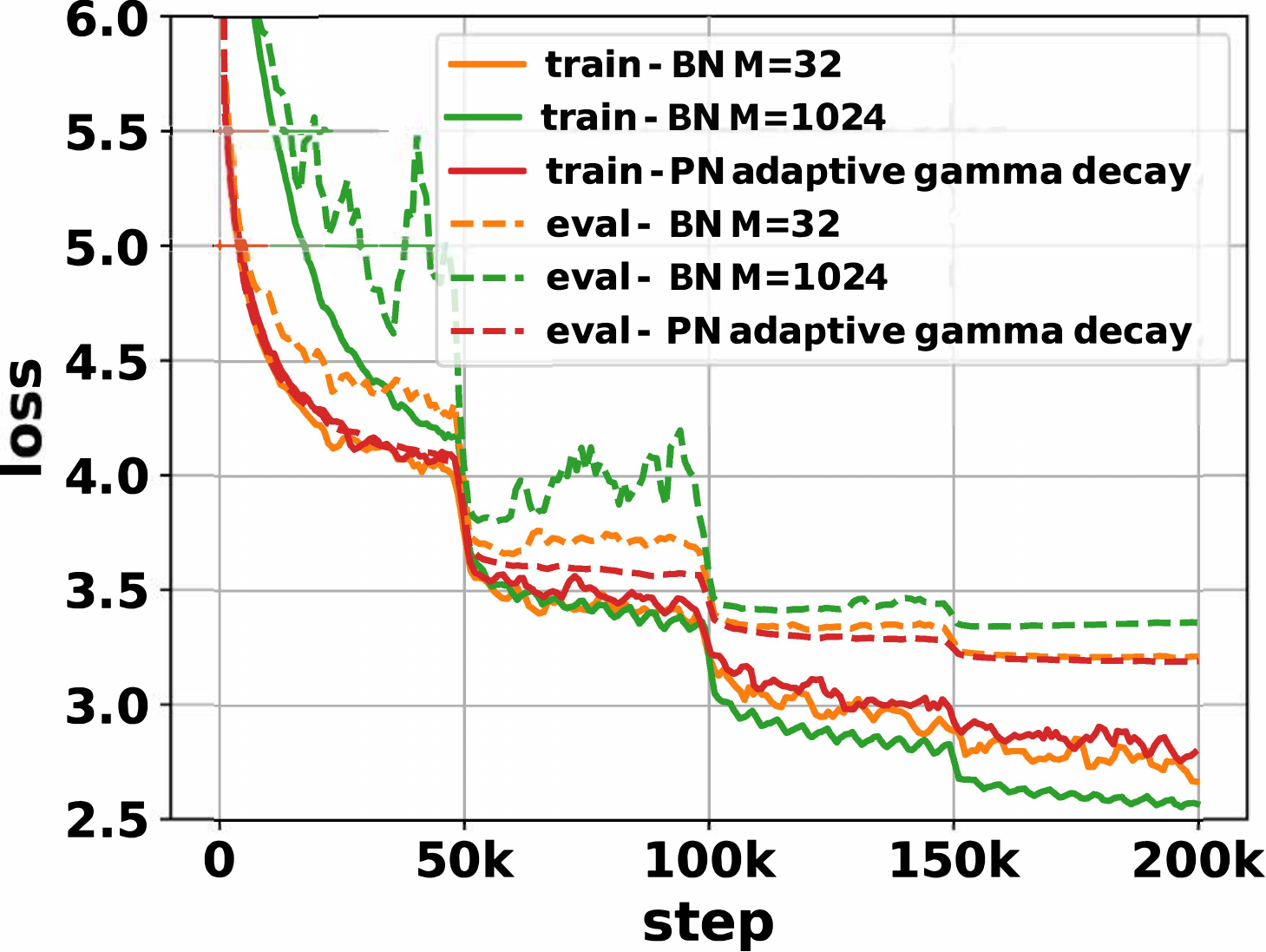}}
\hspace{20pt}
\subfigure[Comparisons of validation accuracy.]{\label{fig:imagenet-b}\includegraphics[width=61mm]{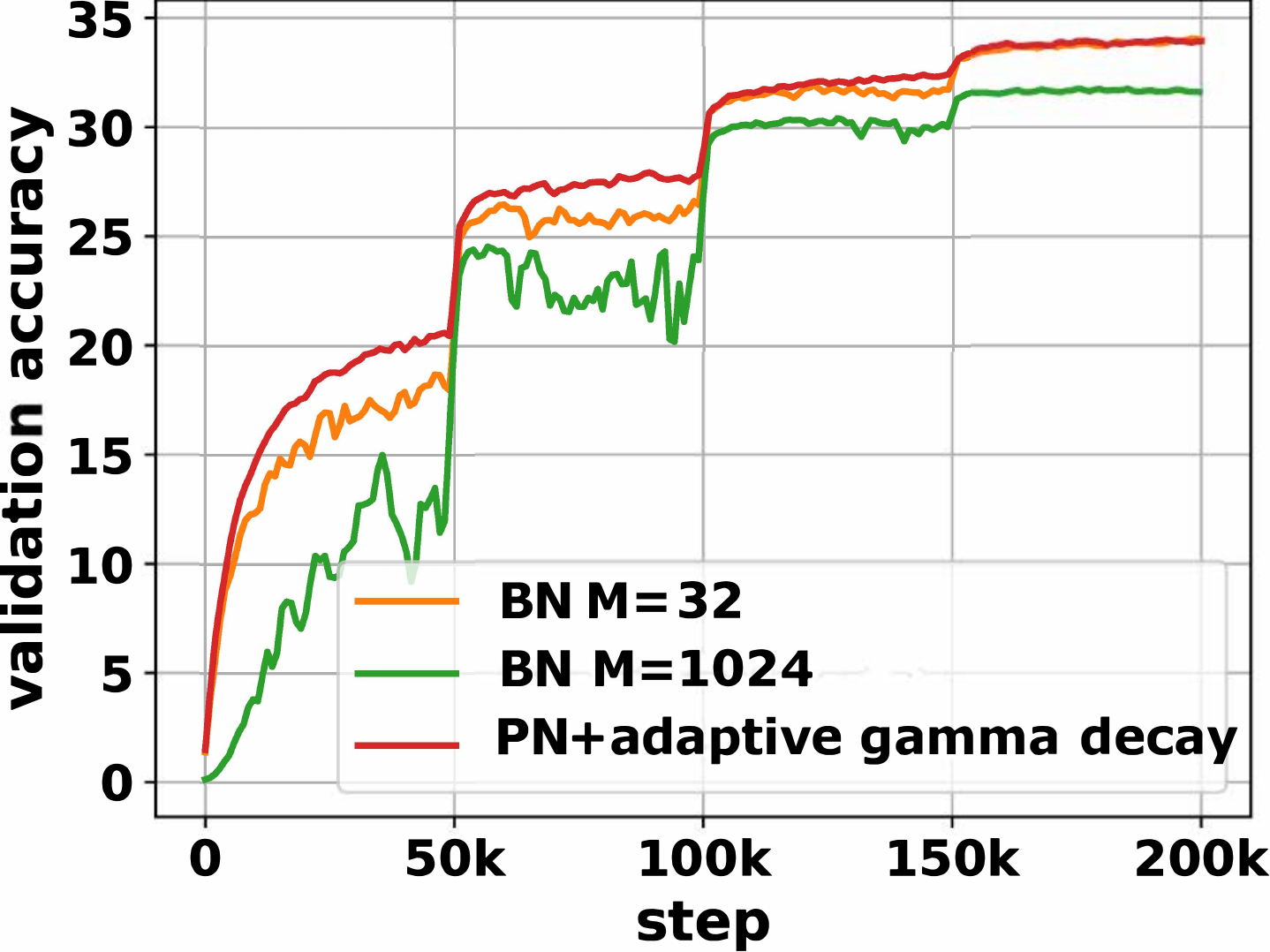}}
%\vspace{-5pt}
   \caption{{\small \textbf{Results of downsampled ImageNet.} (a) plots training and evaluation loss. (b) shows validation accuracy. The models are trained on 8 GPUs.}
   %`BN+M=32' trains on 8 GPUs (32 samples/GPU) with $\mathrm{init\_lr}=0.1$, and `BN+M=8192' on a single GPU with $\mathrm{init\_lr}=3.2$ \citep{goyal_accurate_2017}.
   }
\label{fig:PN_imagenet}
\end{figure}

\subsubsection{Impact of BN to the Norm of Parameters}\label{app:paramnorm}

We demonstrate the impact of BN to the norm of parameters.
%
%in BN is less straightforward in deep models, which
%%
%As deep models are known to suffer from multiple local minima, in order
We compare BN with vanilla SGD, where
a network is first trained by BN in order to converge to a local minima when the parameters do not change much. At this local minima, the weight vector is frozen and denoted as $\mathbf{w}^{\bn}$.
Then this network is finetuned by using vanilla SGD with a small learning rate $10^{-3}$ with the kernel parameters initialized by $\mathbf{w}^\ord=\gamma\frac{\mathbf{w}^{\bn}}{{\sigma}}$, where ${\sigma}$ is the moving average of $\sigma_{\mathcal{B}}$.

Fig.\ref{fig:cifar_finetune} below visualizes the results.
As $\mu_{\mathcal{B}}$ and $\sigma_{\mathcal{B}}$ are removed in the vanilla SGD, it is found from the last two figures that the training loss decreases while the validation loss increases, meaning that the reduction in regularization makes the network converged to a sharper local minimum that generalizes less well.
The magnitudes of kernel parameters $\w^\ord$ at different layers are also displayed in the first four figures. All of them increase after freezing BN, due to the release of regularization on these parameters.

\begin{figure}[h]
    \begin{center}
        \includegraphics[width=1.0\textwidth]{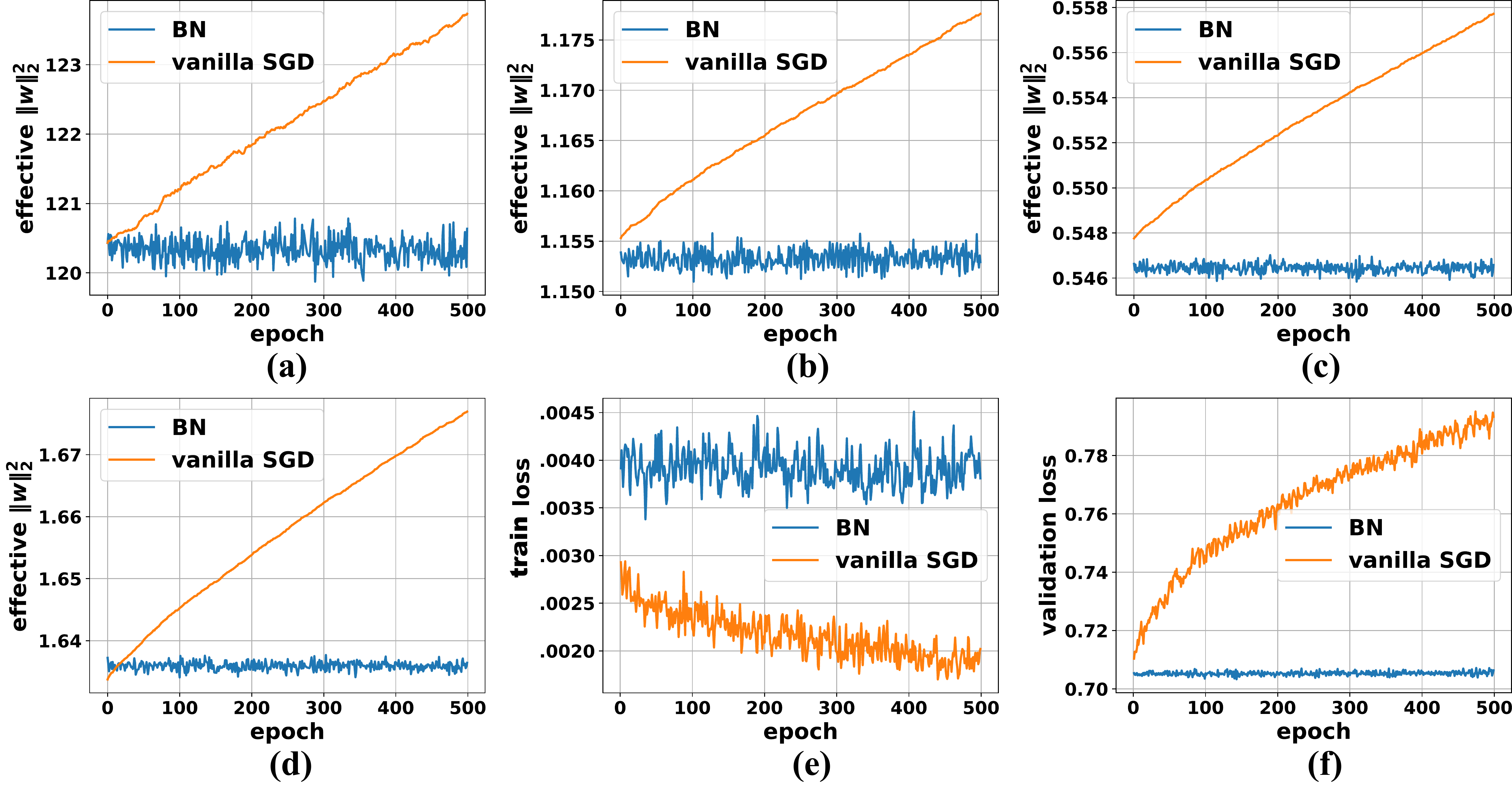}
    \end{center}
    %\vspace{-10pt}
    \caption{\small{\textbf{Study of parameter norm.} Vanilla SGD is finetuned from a network pretrained by BN on CIFAR10. The first four figures show the magnitude of the kernel parameters in different layers in finetuning, compared to the effective norm of BN defined as $\gamma\frac{\|\w\|}{\sigma_\mathcal{B}}$. The last two figures compare the training and validation losses in finetuning.}}
    \label{fig:cifar_finetune}
\end{figure}

\subsubsection{BN and WN with dropout}\label{app:bn-wn-dropout}

\textbf{BN+dropout}. Despite the better generalization of BN with smaller batch sizes, large-batch training is more efficient in real cases. Therefore, improving generalization of BN with large batch is more desiring.
However, gamma decay requires estimating the population statistics that increases computations. We also found that treating the decay factor as a constant hardly improves generalization for large batch.
%
%A closer look at Eqn.(\ref{eq:reg2}) reveals that BN induces an input-data-dependent regularization coefficient $\zeta$, which is difficult to account for in deep networks.
%
Therefore, we utilize dropout as an alternative to compensate for the insufficient regularization. Dropout has also been analytically viewed as a regularizer \citep{wager_dropout_2013}.
% and share a similar expression with Eqn.(\ref{eq:reg1}).
%
We add a dropout after each BN layer to impose regularization.

Fig.\ref{fig:FIGURE2}(g\&h) in the main paper plot the classification results using ResNet18.
%Due to the absence of other regularization methods including data augmentation and weight decay,
The generalization of BN deteriorates significantly when $M$ increases from 128 to 1024. This is observed by the much higher validation loss (Fig.\ref{fig:FIGURE2}(g)) and lower validation accuracy (Fig.\ref{fig:FIGURE2}(h)) when $M=1024$.
If a dropout layer with ratio $0.1$ is added \emph{after} each residual block layer for $M=1024$ in ResNet18, the validation loss is suppressed and accuracy increased by a great margin. This superficially contradicts with the original claim that BN reduces the need for dropout \citep{BN}. We find that there are two differences between our study and \citep{BN}.

First, in pervious study the batch size was fixed at a quite small value (\eg 32), at which the regularization was already quite strong. Therefore, an additional dropout could not further cause better regularization, but on the contrary increases the instability in training and yields a lower accuracy.
However, our study explores relatively large batch that degrades the regularization of BN, and thus dropout with a small ratio can complement.
Second, usual trials put dropout before BN and cause BN to have different variances during training and test. In contrast, dropout follows BN in this study and the distance between two dropout layers is large (a residual block separation), thus the problem can be alleviated. The improvement by applying dropout after BN has also been observed by a recent work \citep{Disharmony}.

\textbf{WN+dropout}. Since BN can be treated as WN trained with regularization as shown in this study, combining WN with regularization should be able to match the performance of BN.
% and the performance of WN might also be similarly improved.
%
As WN outperforms BN in running speed (without calculating statistics) and it suits better in RNNs than BN, an improvement of its generalization is also of great importance.
Fig.\ref{fig:FIGURE2}(g\&h) also show that WN can also be regularized by dropout.
%as above.
%
We apply dropout after each WN layer with ratio 0.2 and the dropout is applied at the same layers as that for BN.
%, which in fact also induces regularization on $\gamma$.
We found that the improvement on both validation accuracy and loss is surprising.
%As shown in Fig.\ref{fig:cifar-b},
The accuracy increases from 0.90 to 0.93, even close to the results of BN. Nevertheless, additional regularization on WN still cannot make WN on par with the performance BN. In deep neural networks the distribution after each layer would be far from a Gaussian distribution, in which case WN is not a good substitute for PN. A potential substibute of BN would require us for designing better estimations of the distribution to improve the training speed and performance of deep networks.

%\begin{figure}[t]
%\centering
%\subfigure[Comparisons of train and validation loss.]{\label{fig:cifar10-pn-a}\includegraphics[width=63mm]{cifar10_loss1}}
%\subfigure[Comparisons of validation accuracy.]{\label{fig:cifar10-pn-b}\includegraphics[width=61mm]{cifar10_acc1}}
%\vspace{-5pt}
%   \caption{{\small\textbf{Results of CIFAR10.} (a) shows training and evaluation loss. (b) plots validation accuracy. The models are trained on a single GPU.}}
%\label{fig:PN_CIFAR}
%\end{figure}

\subsection{Proof of Results}

\subsubsection{Proof of Eqn.\eqref{eq:theorem1}}\label{app:theorem}
%\secindex{app:theorem}

\begin{customthm}{1} [Regularization of $\mu_\mathcal{B},\sigma_\mathcal{B}$]\label{theorem:reg_app}
Let a single-layer perceptron with BN and ReLU activation function be defined by $y=\max(0,\hh),~\hh=\gamma\frac{h-\mu_\mathcal{B}}{\sigma_\mathcal{B}}+\beta~\mathrm{and}~h=\w\tran\x$, where $\x$ and $y$ are the network input and output respectively, $h$ and $\hat{h}$ are the hidden values before and after batch normalization, and $\w$ is the weight vector.
%
%Let the activation function $g(\cdot)$ be ReLU.
Let $\ell(\hat{h})$ be the loss function.
Then
\begin{eqnarray} %\label{eq:theorem1}
%\mathbb{E}_{(x,y)\sim p_{xy}}
\frac{1}{P}\sum_{j=1}^P\mathbb{E}_{\mu_{\mathcal{B}},\sigma_{\mathcal{B}}}\ell(\hat{h}^j)\simeq
%\mathbb{E}_{xy\simp_{xy}}
\frac{1}{P}\sum_{j=1}^P\ell(\bar{h}^j)+\zeta(h)\gamma^2
~~\mathrm{and}~~\zeta(h)={\frac{\rho+2}{8M}\mathcal{I}(\gamma)}+{\frac{1}{2M}\frac{1}{P}\sum_{j=1}^P\sigma(\bar{h}^j)},\nonumber
\end{eqnarray}
where $\bar{h}^j=\gamma\frac{\w\tran\x^j-\mu_{\mathcal{P}}}{\sigma_{\mathcal{P}}}+\beta$ represents population normalization (PN), $\zeta(h)\gamma^2$ represents gamma decay and $\zeta(h)$ is a data-dependent decay factor. $\rho$ is the kurtosis of the distribution of $h$, $\mathcal{I}(\gamma)$ is an estimation of the Fisher information of $\gamma$ and $\mathcal{I}(\gamma)=\frac{1}{P}\sum_{j=1}^P(\frac{\partial\ell(\hat{h}^j)}{\partial\gamma})^2$, and $\sigma(\cdot)$ is a sigmoid function.
\end{customthm}

\begin{proof}
We have $\hat{h}^j=\gamma\frac{\mathbf{w}^{T}\mathbf{x}^j-\mu_{\mathcal{B}}}{\sigma_{\mathcal{B}}}+\beta$
and $\bar{h}^j=\gamma\frac{\mathbf{w}^{T}\mathbf{x}^j-\mu_{\mathcal{P}}}{\sigma_{\mathcal{P}}}+\beta$.
We prove theorem \ref{theorem:reg_app} by performing a Taylor expansion on a function
$A(\hat{h}^j)$ at $\bar{h}^j$, where $A(\hat{h}^j)$ is a function of $\hat{h}^j$ defined according to a particular activation function. The negative log likelihood function of the above single-layer perceptron can be generally defined as $-\log p(y^j|\hat{h}^j)=A(\hat{h}^j)-y^j\hat{h}^j$, which is similar to the loss function of the generalized linear models with different activation functions.
Therefore, we have
\begin{align*}
\frac{1}{P}\sum_{j=1}^{P}\mathbb{E}_{\mu_{\mathcal{B}},\sigma_{\mathcal{B}}}[l(\hat{h}^{j})] & =\frac{1}{P}\sum_{j=1}^{P}\mathbb{E}_{\mu_{\mathcal{B}},\sigma_{\mathcal{B}}}\left[A(\hat{h}^{j})-y^{j}\hat{h}^{j}\right]\\
 & =\frac{1}{P}\sum_{j=1}^{P}(A(\bar{h}^j)-y^{j}\bar{h}^j)+\frac{1}{P}\sum_{j=1}^{P}\mathbb{E}_{\mu_{\mathcal{B}},\sigma_{\mathcal{B}}}\left[-y^{j}(\hat{h}^{j}-\bar{h}^j)+A(\hat{h}^{j})-A(\bar{h}^j)\right]\\
 & =\frac{1}{P}\sum_{j=1}^{P}l(\bar{h}^j)+\frac{1}{P}\sum_{j=1}^{P}\mathbb{E}_{\mu_{\mathcal{B}},\sigma_{\mathcal{B}}}\left[(A^{\prime}(\bar{h}^j)-y^{j})(\hat{h}^{j}-\bar{h}^j)\right]\\
 & +\frac{1}{P}\sum_{j=1}^{P}\mathbb{E}_{\mu_{\mathcal{B}},\sigma_{\mathcal{B}}}\left[\frac{A^{\prime\prime}(\bar{h}^j)}{2}(\hat{h}^{j}-\bar{h}^j)^{2}\right]\\
 & =\frac{1}{P}\sum_{j=1}^{P}l(\bar{h}^j)+R^{f}+R^{q},
\end{align*}
where $A^{\prime}(\cdot)$ and $A^{\prime\prime}(\cdot)$ denote the
first and second derivatives of function $A(\cdot)$. The first and
second order terms in the expansion are represented by $R^{f}$ and
$R^{q}$ respectively. To derive the analytical forms of $R^{f}$
and $R^{q}$, we take a second-order Taylor expansion of of $\frac{1}{\sigma_{\mathcal{B}}}$
and $\frac{1}{\sigma_{\mathcal{B}}^{2}}$ around $\sigma_{P}$, it
suffices to have
\[
\frac{1}{\sigma_{\mathcal{B}}}\approx\frac{1}{\sigma_{\mathcal{P}}}+(-\frac{1}{\sigma_{\mathcal{P}}^{2}})(\sigma_{\mathcal{B}}-\sigma_{\mathcal{P}})+\frac{1}{\sigma_{\mathcal{P}}^{3}}(\sigma_{\mathcal{B}}-\sigma_{\mathcal{P}})^{2}
\]
and
\[
\frac{1}{\sigma_{\mathcal{B}}^{2}}\approx\frac{1}{\sigma_{\mathcal{P}}^{2}}+(-\frac{2}{\sigma_{\mathcal{P}}^{3}})(\sigma_{\mathcal{B}}-\sigma_{\mathcal{P}})+\frac{3}{\sigma_{\mathcal{P}}^{4}}(\sigma_{\mathcal{B}}-\sigma_{\mathcal{P}})^{2}.
\]

By applying the distributions of $\mu_{\mathcal{B}}$ and $\sigma_{\mathcal{B}}$ introduced in section \ref{sec:view}, we have $\mu_{\mathcal{B}}\sim\mathcal{N}(\mu_{\mathcal{P}},\frac{\sigma_{P}^{2}}{M})$ and $\sigma_{\mathcal{B}}\sim\mathcal{N}(\sigma_{P},\frac{\rho+2}{4M})$. Hence, $R^{f}$ can be derived as
%By applying the distributions of $\mu_{\mathcal{B}}$ and $\sigma_{\mathcal{B}}$
in the paper, $R^{f}$ can be derived as
\begin{align*}
R^{f} & =\frac{1}{P}\sum_{j=1}^{P}\mathbb{E}_{\mu_{\mathcal{B}},\sigma_{\mathcal{B}}}\left[(A^{\prime}(\bar{h}^j)-y^{j})(\hat{h}^{j}\bar{h}^j)\right]\\
 & =\frac{1}{P}\sum_{j=1}^{P}\mathbb{E}_{\mu_{\mathcal{B}},\sigma_{\mathcal{B}}}\left[(A^{\prime}(\bar{h}^j)-y^{j})\left(\gamma\frac{\mathbf{w}^{T}\mathbf{x}^{j}-\mu_{\mathcal{B}}}{\sigma_{\mathcal{B}}}-\gamma\frac{\mathbf{w}^{T}\mathbf{x}^{j}-\mu_{\mathcal{P}}}{\sigma_{\mathcal{P}}}\right)\right]\\
 & =\frac{1}{P}\sum_{j=1}^{P}\mathbb{E}_{\mu_{\mathcal{B}},\sigma_{\mathcal{B}}}\left[(A^{\prime}(\bar{h}^jy^{j})\left(\gamma\mathbf{w}^{T}\mathbf{x}^{j}\left(\frac{1}{\sigma_{\mathcal{B}}}-\frac{1}{\sigma_{\mathcal{P}}}\right)+\gamma\left(-\frac{\mu_{\mathcal{B}}}{\sigma_{\mathcal{B}}}+\frac{\mu_{\mathcal{P}}}{\sigma_{\mathcal{P}}}\right)\right)\right]\\
 & =\frac{1}{P}\sum_{j=1}^{P}\gamma(A^{\prime}(\bar{h}^j)-y^{j})(\mathbf{w}^{T}\mathbf{x}^{j}-\mu_{\mathcal{P}})\mathbb{E}_{\sigma_{\mathcal{B}}}\left[\frac{1}{\sigma_{\mathcal{B}}}-\frac{1}{\sigma_{\mathcal{P}}}\right]\\
 & =\frac{1}{P}\sum_{j=1}^{P}\frac{\rho+2}{4M}\gamma(A^{\prime}(\bar{h}^j)-y^{j})\frac{\mathbf{w}^{T}\mathbf{x}^{j}-\mu_{\mathcal{P}}}{\sigma_{\mathcal{P}}}.
\end{align*}
This $R^{f}$ term can be understood as below. Let $h=\frac{\mathbf{w}^{T}\mathbf{x}-\mu_{\mathcal{P}}}{\sigma_{\mathcal{P}}}$ and the distribution of the population data be $p_{xy}$. We establish the following relationship
\begin{align*}
\mathbb{E}_{(x,y)\sim p_{xy}}\mathbb{E}_{\mu_{\mathcal{B}},\sigma_{\mathcal{B}}}\left[(A^{\prime}(\bar{h})-y)h\right]&=\mathbb{E}_{\mu_{\mathcal{B}},\sigma_{\mathcal{B}}}\mathbb{E}_{x\sim p_{x}}\mathbb{E}_{y|x\sim p_{y|x}}\left[(A^{\prime}(\bar{h})-y)h\right]\\
 & =\mathbb{E}_{\mu_{\mathcal{B}},\sigma_{\mathcal{B}}}\mathbb{E}_{x\sim p_{x}}\left[(\mathbb{E}\left[y|x\right]-\mathbb{E}_{y|x\sim p_{y|x}}\left[y\right])h\right]\\
 & =0.
\end{align*}

Since the sample mean converges in probability to the population mean by the Weak Law of Large Numbers, for all $\epsilon>0$ and a constant number $K$ ($\exists K>0$~and~$\forall P>K$), we have $p\left(\big|R^f-\mathbb{E}_{(x,y)\sim p_{xy}}\mathbb{E}_{\mu_{\mathcal{B}},\sigma_{\mathcal{B}}}\left[(A^{\prime}(\bar{h})-y)h\right]\big|\geq \frac{\rho+2}{4M}\epsilon\right)=0$. This equation implies that $R^f$ is sufficiently small with a probability of 1 given moderately large number of data points $P$ (the above inequality holds when $P>30$).

On the other hand, $R^{q}$ can be derived as
\begin{align*}
R^{q} & =\frac{1}{P}\sum_{j=1}^{P}\mathbb{E}_{\mu_{\mathcal{B}},\sigma_{\mathcal{B}}}\left[\frac{A^{\prime\prime}(\bar{h}^j)}{2}(\hat{h}^{j}-\bar{h}^j)^{2}\right]\\
 & =\frac{1}{P}\sum_{j=1}^{P}\frac{A^{\prime\prime}(\bar{h}^j)}{2}\mathbb{E}_{\mu_{\mathcal{B}},\sigma_{\mathcal{B}}}\left[(\gamma\frac{\mathbf{w}^{T}\mathbf{x}^{j}-\mu_{\mathcal{B}}}{\sigma_{\mathcal{B}}}+\beta-\gamma\frac{\mathbf{w}^{T}\mathbf{x}^{j}-\mu_{\mathcal{P}}}{\sigma_{\mathcal{P}}}+\beta)^{2}\right]\\
 & =\frac{1}{P}\sum_{j=1}^{P}\frac{A^{\prime\prime}(\bar{h}^j)}{2}\mathbb{E}_{\mu_{\mathcal{B}},\sigma_{\mathcal{B}}}\left[(\gamma\mathbf{w}^{T}\mathbf{x}^{j})^{2}(\frac{1}{\sigma_{\mathcal{B}}}-\frac{1}{\sigma_{\mathcal{P}}})^{2}-2\gamma\mu_{\mathcal{P}}\mathbf{w}^{T}\mathbf{x}^{j}(\frac{1}{\sigma_{\mathcal{B}}}-\frac{1}{\sigma_{\mathcal{P}}})^{2}+(\frac{\mu_{\mathcal{B}}}{\sigma_{\mathcal{B}}}-\frac{\mu_{\mathcal{P}}}{\sigma_{\mathcal{P}}})^{2}\right]\\
 & \simeq\frac{1}{P}\sum_{j=1}^{P}\frac{\gamma^{2}A^{\prime\prime}(\bar{h}^j)}{2}\left((\mathbf{w}^{T}\mathbf{x}^{j}-\mu_{\mathcal{P}})^{2}\mathbb{E}_{\mu_{\mathcal{B}},\sigma_{\mathcal{B}}}\left[(\frac{1}{\sigma_{\mathcal{B}}}-\frac{1}{\sigma_{\mathcal{P}}})^{2}\right]+\mathbb{E}_{\mu_{\mathcal{B}},\sigma_{\mathcal{B}}}\left[\left(\frac{\mu_{\mathcal{B}}-\mu_{P}}{\sigma_{\mathcal{B}}}\right)^{2}\right]\right)\\
 & =\frac{1}{P}\sum_{j=1}^{P}\frac{\gamma^{2}A^{\prime\prime}(\bar{h}^j)}{2}\left((\frac{\mathbf{w}^{T}\mathbf{x}^{j}-\mu_{\mathcal{P}}}{\sigma_{\mathcal{P}}})^{2}\frac{\rho+2}{4M}+\frac{1}{M}(1+\frac{3(\rho+2)}{4M})\right).\label{R^q:last}
\end{align*}

Note that
$\frac{\partial^2 l(\bar{h}^j)}{\partial\gamma^2}=A^{\prime\prime}(\bar{h}^j)(\frac{\mathbf{w}^{T}\mathbf{x}^j-\mu_{\mathcal{P}}}{\sigma_{\mathcal{P}}})^{2}$, we have $\mathcal{I}(\gamma)=\frac{1}{P}\sum_{j=1}^{P}A^{\prime\prime}(\bar{h}^j)(\frac{\mathbf{w}^{T}\mathbf{x}^j-\mu_{\mathcal{P}}}{\sigma_{\mathcal{P}}})^{2}$ been an estimator of the Fisher information with respect to the scale parameter $\gamma$. %according to the definition of Fisher information.
Then, by neglecting $O(1/M^{2})$ high-order term in $R^q$, we get
\[
R^{q}\simeq\frac{\rho+2}{8M}\mathcal{I}(\gamma)\gamma^{2}+\frac{\mu_{d^{2}A}}{2M}\gamma^{2}, \label{R^q:two}
\]
where $\mu_{d^{2}A}$ indicates the mean
of the second derivative of $A(h)$.
\end{proof}

The results of both ReLU activation function and identity function are provided as below.

\subsubsection{ReLU Activation Function}\label{app:theorem-relu}

For the ReLU non-linear activation function, that is $g(h)=\max(h,0)$, we use its continuous approximation softplus function $g(h)=\log(1+\exp(h))$ to derive the partition function $A(h)$.
%\textbf{Linear.} For a linear activation function $f(h)=h$ with input $\x$
% meets specific conditions.
In this case, we have $\mu_{d^{2}A}=\frac{1}{P}\sum_{j=1}^{P}\sigma(\bar{h}^j)$.
Therefore, we have $\zeta(h)=\frac{\rho+2}{8M}\mathcal{I}(\gamma)+\frac{1}{2M}\frac{1}{P}\sum_{j=1}^P\sigma(\bar{h}^j)$ as shown in Eqn.\eqref{eq:theorem1}.

\subsubsection{Linear Student Network with Identity Activation Function}\label{app:theorem-id}

For a loss function with identity (linear) units, $\frac{1}{P}\sum_{j=1}^{P}\big({\w^{\ast}}\tran\x^j-\gamma(\w\tran\x^j-\mu_\mathcal{B})/\sigma_{\mathcal{B}}\big)^{2}$, we have $\mathcal{I}(\gamma)=2\lambda$ and
%the regularization contribution from $\mu_\mathcal{B}$ can be neglected. We also have
$\rho=0$ for Gaussian input distribution.
%
%Therefore, $\zeta=\frac{\lambda}{2M}$ when $M>32$.
The exact expression of Eqn.\eqref{eq:theorem1} is also possible for such linear regression problem. Under the condition of Gaussian input $\x \sim \mathcal{N}(0,1/N)$, $h=\w\tran\x$ is also a random variable satisfying a normal distribution $~\mathcal{N}(0,1)$. It can be derived that $\mathbb{E}\left(\sigma_{\mathcal{B}}^{-1}\right)  =\frac{\sqrt{M}}{\sqrt{2}\sigma_{\mathcal{P}}}\frac{\Gamma\left(\frac{M-2}{2}\right)}{\Gamma\left(\frac{M-1}{2}\right)}$ and $\mathbb{E}\left(\sigma_{\mathcal{B}}^{-2}\right)  =\frac{M}{\sigma_{\mathcal{P}}^{2}}\frac{\Gamma\left(\frac{M-1}{2}-1\right)}{\Gamma\left(\frac{M-1}{2}\right)}$. Therefore
\begin{align*}
 \zeta=\lambda \left(1+\frac{M\Gamma\big((M-3)/2\big)}{2\Gamma\big((M-1)/2\big)} - \sqrt{2M}\frac{\Gamma\big((M-2)/2\big)}{\Gamma\big((M-1)/2\big)}\right).
\end{align*}

Furthermore, the expression of $\zeta$ can be simplified as $\zeta=\frac{3}{4M}$.
%, where $M$ is the batch size. Furthermore,
If the bias term is neglected in a simple linear regression, contributions from
$\mu_\mathcal{B}$ to the regularization term is neglected and thus $\zeta=\frac{1}{4M}$. Note that if one uses mean square error without being divided by 2 during linear regression, the values for $\zeta$ should be multiplied by 2 as well, where $\zeta=\frac{1}{2M}$.

\subsubsection{BN Regularization in a Deep Network}\label{app:deep-reg}

The previous derivation is based on the single-layer perceptron. In deep neural networks, the forward computation inside one basic building block of a deep network is written by
\begin{equation}%\label{eq:BN}
\small
z_i^{l}=g(\hh_i),~~~\hh_i^l={\gamma}_i^l\frac{{h}_i^l-(\mu_\mathcal{B})_i^l}{(\sigma_B)_i^l}+\beta_i^l ~~~\mathrm{and}~~~h_i^l=(\mathbf{w}_i^l)\tran \mathbf{z}^{l-1},
\end{equation}
where the superscript $l\in [1, L]$ is the index of a building block in a deep neural network, and $i\in[1,N^l]$ indexes each neuron inside a layer. $z^0$ and $z^{L}$ are synonyms of input $x$ and output $y$, respectively. In order to analyze the regularization of BN from a specific layer, one needs to isolate its input and focus on the noise introduced by the BN layer in this block.
Therefore, the loss function $\ell(\hat{h}^l)$ can also be expanded at $\ell(\bar{h}^l)$. In BN, the batch variance is calculated with regard to each neuron under the assumption of mutual independence of neurons inside a layer. By following this assumption and the above derivation in Appendix \ref{app:theorem}, the loss function with BN in deep networks can also be similarly decomposed.

\textbf{Regularization of $\mu_\mathcal{B}^l,\sigma_\mathcal{B}^l$ in a deep network.}
Let $\mathbf{\zeta}^l$ be the strength (coefficient) of the regularization at the $l$-{{th}} layer. Then
\begin{eqnarray} %\label{eq:theorem1}
%\mathbb{E}_{(x,y)\sim p_{xy}}
&&\frac{1}{P}\sum_{j=1}^P\mathbb{E}_{\mu_{\mathcal{B}}^l,\sigma_{\mathcal{B}}^l}\ell\big((\hat{h}^l)^j\big)\simeq
%\mathbb{E}_{xy\simp_{xy}}
\frac{1}{P}\sum_{j=1}^P\ell\big((\bar{h}^l)^j\big)+\sum_i^{N^l}{\zeta_i^l \cdot (\mathbf{\gamma}_i^l)^2},\nonumber\\
&&~~
\mathrm{and}~~ {\zeta}_i^l =\frac{1}{P}\sum_{j=1}^{P}\frac{\mathrm{diag}\big(\mathcal{H}_{\ell}(\bar{h}^l)^j\big)_i}{2} \left(\frac{\rho_i^l+2}{4M}\bigg(\frac{(\mathbf{w}_i^l)^{T}(\mathbf{z}^{l-1})^{j}-(\mu_{\mathcal{P}})_i^l}{(\sigma_{\mathcal{P}})_i^l}\bigg)^{2}+\frac{1}{M}\right) +\mathcal{O}(1/M^2), \nonumber
\end{eqnarray}
where $i$ is the index of a neuron in the layer, $(\bar{h}_i^l)^j=\gamma_i^l \frac{(\mathbf{w}_i^l)^{T}(\mathbf{z}^{l-1})^{j}-(\mu_{\mathcal{P}})_i^l}{(\sigma_{\mathcal{P}})_i^l}+\beta_i^l$ represents population normalization (PN), $\mathcal{H}_{\ell}(\bar{h}^l)$ is the Hessian matrix at $\bar{h}^l$ regarding to the loss $\ell$ and $\mathrm{diag}(\cdot)$ represents the diagonal vector of a matrix.

It is seen that the above equation is compatible with the results from the single-layer perceptron. The main difference of the regularization term in a deep model is that the Hessian matrix is not guaranteed to be positive semi-definite during training. However, this form of regularization is also seen from other regularization such as noise injection \citep{rifai_adding_2011} and dropout \citep{wager_dropout_2013}, and has long been recognized as a Tikhonov regularization term \citep{bishop_training_1995}.

In fact, it has been reported that in common neural networks, where convex activation functions such as ReLU and convex loss functions such as common cross entropy are adopted, the Hessian matrix $\mathcal{H}_\ell({\bar{h}^l})$ can be seen as `locally' positive semidefinite \citep{santurkar_how_2018}. Especially, as training converges to its mimimum training loss, the Hessian matrix of the loss can be viewed as positive semi-definite and thus the regularization term on $\gamma^l$ is positive.

%Through the course of proof, we can easily find that the first and second term in Eqn.\ref{R^q:two} derived from the expectation on $\sigma_{\mathcal{B}}$ and $\mu_{\mathcal{B}}$ respectively.
%
%

\subsubsection{Dynamical Equations}\label{app:dyn}
%\secindex{app:dyn}

Here we discuss the dynamical equations of BN. Let the length of teacher's weight vector be 1, that is, $\frac{1}{N}\mathbf{w^{\ast}}\tran\mathbf{w^{\ast}}=1$.
%
%As for the overlap between student weight vector and teacher weightvector and the length of student weight vector, we set them as $\frac{1}{N}\mathbf{w}^{T}\mathbf{w^{\ast}}=QR$and $\frac{1}{N}\mathbf{w}^{T}\mathbf{w}=Q^{2}$ in ordinary network.
We introduce a normalized weight vector of the student as $\mathbf{\widetilde{w}}=\sqrt{N}\gamma\frac{\mathbf{w}}{\left\Vert \mathbf{w}\right\Vert }$.
Then the overlapping ratio between teacher and student, the length of student's vector, and the length of student's normalized weight vector are $\frac{1}{N}\mathbf{\widetilde{w}}\tran\mathbf{w^{\ast}}=QR=\gamma R$,
$\frac{1}{N}\mathbf{\widetilde{w}}\tran\widetilde{\mathbf{w}}=Q^{2}=\gamma^2$, and
$\frac{1}{N}\mathbf{w}\tran\mathbf{w}=L^{2}$ respectively, where $Q=\gamma$.
And we have $\frac{1}{N}\mathbf{w}\tran\w=LR$.

We transform update equations \eqref{eq:w} by using order parameters.
The update rule for variable $Q^2$ can be obtained by
$\big(Q^2\big)^{j+1}-\big(Q^2\big)^{j}=\frac{1}{N}\big[2\eta{\delta^j\swj}\tran\x^j
-2\eta\zeta\big(Q^2\big)^j\big]$ following update rule of $\gamma$.
Similarly, the update rules for variables $RL$ and $L^2$ are calculated as follow:
\begin{equation}\label{eq:dRL}
\begin{split}
&\big(RL\big)^{j+1}-\big(RL\big)^{j}=\frac{1}{N }\big(\frac{\eta Q^j}{L^j}\delta^j{{\w^\ast}}\tran\x^j-\frac{\eta R^j}{L^j}\delta^j\swj\tran\x^j\big),\\
&\big({L}^2\big)^{j+1}-\big(L^2\big)^{j}=\frac{1}{N}
\big[\frac{\eta^2(Q^2)^j}{(L^2)^j}{\delta^j}^2{\x^j}\tran\x^{j}
-\frac{\eta^2}{N(L^2)^{j}}{\delta^j}^2(\swj\tran\x^{j})^{2}\big].
\end{split}
\end{equation}

Let $t=\frac{j}{N}$ is a normalized sample index that can be treated as a continuous time variable.
We have $\Delta t=\frac{1}{N}$ that approaches zero in the thermodynamic limit when $N\rightarrow\infty$.
In this way, the learning dynamic of $Q^2$, $RL$ and $L^2$ can be formulated as the following differential equations:
\begin{equation}\label{eq:cldy}
\left\{
\begin{array}{lll}
\frac{dQ^2}{dt}&=2\eta I_{1}-2\eta\zeta Q^2,\\
\frac{dRL}{dt}&=\eta\frac{Q}{L}I_{3}-\eta\frac{R}{L}I_{1},\\
\frac{dL^2}{dt}&=\eta^{2}\frac{Q^{2}}{L^{2}}I_{2},
\end{array}
\right.
\end{equation}

where $I_1=\langle\delta{\sw}\tran\x\rangle_\x$, $I_2=\langle\delta^2\x\tran\x\rangle_\x$, and $I_3=\langle\delta{\w^\ast}\tran\x\rangle_\x$, which are the terms presented in $\frac{dQ^2}{dt}$, $\frac{dRL}{dt}$, and $\frac{d{L}^2}{dt}$ and $\langle\cdot\rangle_\x$ denotes expectation over the distribution of $\x$. They are used to simplify notations.
Note that we neglect the last term of $dL^2/dt$ in Eqn.(\ref{eq:dRL}) since $\frac{\eta^2}{N(L^2)}{\delta}^2(\sw\tran\x)^{2}$ can be approximately equal to zero when $N$ approaches infinity.
On the other hand, we have $dQ^2=2QdQ, dRL=RdL+LdR$ and $dL^2=2LdL$. Hence, Eqn.(\ref{eq:cldy}) can be reduced to
\begin{equation}\label{eq:QRL}
\left\{
\begin{array}{lll}
\frac{dQ}{dt}&=\eta\frac{I_{1}}{Q}-\eta\zeta Q,\\
\frac{dR}{dt}&=\eta\frac{Q}{L^{2}}I_{3}-\eta\frac{R}{L^{2}}I_{1}-\eta^{2}\frac{Q^{2}R}{2L^{4}}I_{2},\\
\frac{dL}{dt}&=\eta^{2}\frac{Q^{2}}{2L^{3}}I_{2}.
\end{array}
\right.
\end{equation}

\vspace{5pt}
\begin{prop}\label{prop:fixp}
Let $(Q_0,R_0,L_0)$ denote a fixed point with parameters $Q$, $R$ and $L$ of Eqn.(\ref{eq:QRL}). Assume the learning rate $\eta$ is sufficiently small when training converges and $x\sim\mathcal{N}(0,\frac{1}{N}\mathbf{I})$. If activation function $g$ is $\mathrm{ReLU}$, then we have $Q_0=\frac{1}{2\zeta+1},R_0=1$ and $L_0$ could be arbitrary.
\end{prop}

\begin{proof}
First, $L$ has no influence on the output of student model since $\w$ is normalized, which implies that if $(Q_0,R_0,L_0)$ is a fixed point of Eqn.(\ref{eq:QRL}),  $L_0$ could be arbitrary. Besides, we have $\eta\gg\eta^2$ because the learning rate $\eta$ is sufficiently small. Therefore, the terms in Eqn.(\ref{eq:QRL}) proportional to $\eta^2$ can be neglected. If $(Q_0,R_0,L_0)$ is a fixed point, it suffices to have

\begin{eqnarray}
\eta\frac{I_{1}(Q_{0},R_{0})}{Q_{0}}-\eta\zeta Q_{0}&=0,\label{eq:fixQ}\\
\eta\frac{Q_{0}}{L_0^{2}}I_{3}(Q_{0},R_{0})-\eta\frac{R_{0}}{L_0^{2}}I_{1}(Q_{0},R_{0})&=0, \label{eq:fixR}
\end{eqnarray}
To calculate $I_1$ and $I_3$, we define $s$ and $t$ as $\sw\tran\x$ and $\w\tran\mathbf{x}$.
Since $\mathbf{x}\sim\mathcal{N}(0,\frac{1}{N}\mathbf{I})$, we can acquire
\[
\left[\begin{array}{c}
s\\
t
\end{array}\right]\sim N\left(\left(\left[\begin{array}{c}
0\\
0
\end{array}\right],\left[\begin{array}{cc}
Q^{2} & QR\\
QR & 1
\end{array}\right]\right)\right)
\]
so probability measure of $[s,t]\tran $ can be written as

\[
DsDt=\frac{1}{2\pi Q\sqrt{1-R^{2}}}exp\left\{ -\frac{1}{2}\left[\begin{array}{c}
s\\
t
\end{array}\right]^{T}\left[\begin{array}{cc}
Q^{2} & QR\\
QR & 1
\end{array}\right]^{-1}\left[\begin{array}{c}
s\\
t
\end{array}\right]\right\}
\]
Then,
\begin{equation}\label{eq:intI1}
\begin{split}
I_{1} & =\left\langle g^{\prime}(\mathbf{\widetilde{w}}\tran \mathbf{x})\left[g(\mathbf{w}^{\ast T}\mathbf{x})-g(\mathbf{\widetilde{w}}\tran \mathbf{x})\right]\mathbf{\widetilde{w}}\tran \mathbf{x}\right\rangle _{\mathbf{x}}\\
 & =\intop_{u,v}[g'(s)\left(g(t)-g(s\right)s]DsDt\\
 & =\int_{0}^{+\infty}\int_{0}^{+\infty}stDsDt-\int_{0}^{+\infty}s^{2}\int_{-\infty}^{+\infty}DsDt\\
 & =\frac{Q(\pi R+2\sqrt{1-R^{2}}+2Rarcsin(R))}{4\pi}-\frac{Q^{2}}{2}
\end{split}
\end{equation}
and
\begin{equation}\label{eq:intI3}
\begin{split}
I_{3} & =\intop_{u,v}[g'(s)\left(g(t)-g(s\right)t]DsDt\\
 & =\intop_{u,v}g'(s)g(t)tDsDt-\intop_{u,v}g'(s)g(s)tDsDt\\
 & =\int_{0}^{+\infty}\int_{0}^{+\infty}t^{2}DsDt-\int_{0}^{+\infty}\int_{-\infty}^{+\infty}stDsDt\\
 & =\frac{\pi+2R\sqrt{1-R^{2}}+2\arcsin(R)}{4\pi}-\frac{QR}{2}
\end{split}
\end{equation}
By substituting Eqn.(\ref{eq:intI1}) and (\ref{eq:intI3}) into Eqn.(\ref{eq:fixQ}) and (\ref{eq:fixR}), we get $Q_{0}=\frac{1}{2\zeta+1}$ and $R_0=1$.
\end{proof}
%\vspace{5pt}
\begin{prop}\label{prop:eigen}
Given conditions in proposition\ref{prop:fixp}, let $\lambda_{Q}^{\bn}$, $\lambda_{R}^{\bn}$
be the eigenvalues of the Jacobian matrix at fixed point $(Q_{0},R_0,L_{0})$
corresponding to the order parameters $Q$ and $R$ respectively in
BN. Then

\[
\begin{cases}
\lambda_{Q}^{\bn}=\frac{\eta}{Q_{0}}\frac{\partial I_{1}}{\partial Q}-\eta\zeta Q_{0},\\
\lambda_{R}^{\bn}=\frac{\partial I_{2}}{2\partial R}\frac{\eta Q_{0}}{2L_{0}^{2}}(\eta_{\mathrm{max}}^{\bn}-\eta_{\mathrm{eff}}^{bn}),
\end{cases}
\]

where $\eta_{\mathrm{max}}^{\bn}$ and $\eta_{\mathrm{eff}}^{\bn}$ are the maximum and
effective learning rates respectively in BN.
\end{prop}

\begin{proof}

%Firstly, note that the change of $L$ will not change $I_{1}$,
%$I_{2}$ and $I_{3}$. Thus we have $\frac{\partial I_{1}}{\partial L}=\frac{\partial I_{2}}{\partial L}=\frac{\partial I_{3}}{\partial L}=0$. And we also neglect the term proportional to $\eta^{2}$ because
%of the learning rate decays to a small value when converged. At fixed point $R_0=1$, we have $\partial(QI_{3}-I_{1})/\partial Q=0$,
At fixed point $(Q_{0},R_0,L_{0})=(\frac{1}{2\zeta+1},1,L_0)$ obtained in proposition\ref{prop:fixp}, the Jacobian of dynamic equations of BN can be derived as
\[
J^{\bn}=\left[\begin{array}{ccc}
\frac{\eta}{Q_{0}}\frac{\partial I_{1}}{\partial Q}-2\eta\zeta & \frac{\eta}{Q_{0}}\frac{\partial I_{1}}{\partial R} & 0\\
0 & \frac{\eta}{L_{0}^{2}}\left(\frac{Q_{0}\partial I_{3}}{\partial R}-\frac{\partial I_{1}}{\partial R}-\zeta Q_{0}^{2}\right)-\frac{\eta^{2}Q_{0}^{2}}{2L_{0}^{4}}\frac{\partial I_{2}}{\partial R} & 0\\
0 & \frac{\eta^{2}Q_{0}^{2}}{2L_{0}^{3}}\frac{\partial I_{2}}{\partial R} & 0
\end{array}\right],
\]
and the eigenvalues of $J^{\bn}$ can be obtained by inspection

\[
\begin{cases}
\lambda_{Q}^{\bn}=\frac{\eta}{Q_{0}}\frac{\partial I_{1}}{\partial Q}-2\eta\zeta,\\
\lambda_{R}^{\bn}=\frac{\eta}{L_{0}^{2}}\left(\frac{Q_{0}\partial I_{3}}{\partial R}-\frac{\partial I_{1}}{\partial R}-\zeta Q_{0}^{2}\right)-\frac{\eta^{2}Q_{0}^{2}}{2L_{0}^{4}}\frac{\partial I_{2}}{\partial R}=\frac{\partial I_{2}}{\partial R}\frac{\eta Q_{0}}{2L_{0}^{2}}\left(\eta_{\mathrm{max}}^{\bn}-\eta_{\mathrm{eff}}^{\bn}\right),\\
\lambda_{L}^{\bn}=0.
\end{cases}
\]

Since $\gamma_{0}=Q_{0}$, we have $\eta_{\mathrm{max}}^{\bn}=(\frac{\partial(\gamma_{0}I_{3}-I_{1})}{\gamma_{0}\partial R}-\zeta\gamma_{0})/\frac{\partial I_{2}}{2\partial R}$
and $\eta_{\mathrm{eff}}^{\bn}=\frac{\eta\gamma_{0}}{L_{0}^{2}}$.
\end{proof}

\subsubsection{stable fixed points of BN}\label{app:lr}

%\vspace{5pt}
\begin{prop}\label{prop:constraint}
Given conditions in proposition\ref{prop:fixp}, when activation function is ReLU, then (i) $\lambda_{Q}^{\bn}<0$, and (ii) $\lambda_{R}^{\bn}<0$
iff $\eta_{\mathrm{max}}^{\bn}>\eta_{\mathrm{eff}}^{\bn}$.
\end{prop}

\begin{proof}
When activation function is ReLU, we derive $I_{1}=\frac{Q(\pi R+2\sqrt{1-R^{2}}+2R\arcsin(R))}{4\pi}-\frac{Q^{2}}{2}$, which gives
\[
\frac{\partial I_{1}}{\partial Q}=-Q+\frac{\pi R+2\sqrt{1-R^{2}}+2R\arcsin(R)}{4\pi}.
\]

Therefore at the fixed point of BN $(Q_{0},R_0,L_{0})=(\frac{1}{2\zeta+1},1,L_0)$, we have
\[
\lambda_{Q}^{\bn}=\eta(\frac{1}{Q_{0}}\frac{\partial I_{1}}{\partial Q}-2\zeta)=\eta(\frac{1}{Q_{0}}(-1+\frac{1}{2Q_{0}}-2\zeta)=-\zeta-\frac{1}{2}<0.
\]

Note that $\mathbf{x}\tran \mathbf{x}$ approximately equals 1.
We get
\begin{equation}\label{eq:intI2}
\begin{split}
I_{2} & =\intop_{u,v}[g'(s)\left(g(t)-g(s)\right)]^{2}DsDt\\
 & =\int_{0}^{+\infty}\int_{0}^{+\infty}v^{2}DsDt+\int_{0}^{+\infty}\int_{-\infty}^{+\infty}s^{2}DsDv -2\int_{0}^{+\infty}\int_{-\infty}^{+\infty}stDsDt\\
 & =\frac{Q^{2}}{2}+\frac{\pi R+2R\sqrt{1-R^{2}}+2\arcsin(R)}{4\pi}-\frac{Q(\pi R+2\sqrt{1-R^{2}}+2R\arcsin(R))}{2\pi}.
\end{split}
\end{equation}

At the fixed point we have $\frac{\partial I_{2}}{\partial R}=-Q_{0}<0$.
Therefore, we conclude that $\lambda_{R}^{\bn}<0$ iff $\eta_{\mathrm{max}}^{\bn}>\eta_{\mathrm{eff}}^{\bn}$.
\end{proof}

\subsubsection{Maximum Learning Rate of BN}\label{app:maxlr}

\begin{prop}\label{prop:maxeta}
When the activation function is ReLU, then $\eta_{\mathrm{max}}^{\bn}\geq\eta_{\mathrm{max}}^{\{\wn,\ord\}}+2\zeta$,
where $\eta_{\mathrm{max}}^{\bn}$ and $\eta_{\mathrm{max}}^{\{\wn,\ord\}}$ indicate the maximum
learning rates of BN, WN, and vanilla SGD respectively.
\end{prop}

\begin{proof}
From the above results, we have $I_{1}=\frac{Q(\pi R+2\sqrt{1-R^{2}}+2R\arcsin(R))}{4\pi}-\frac{Q^{2}}{2}$,
which gives $\partial I_{1}/\partial R\geq0$ at the fixed point of BN.
Then it can be derived that $\frac{\partial I_{2}}{\partial R}<0$.
%and $\frac{\partial(I_{3}-I_{1})}{\partial R}/\frac{\partial I_{2}}{2\partial R}$
%has the same value at their respective fixed point of BN, WN and vanilla SGD.
Furthermore, at the fixed point of BN, $Q_{0}=\gamma_{0}=\frac{1}{2\zeta+1}<1$,
then we have
\begin{align*}
\eta_{\mathrm{max}}^{\bn} & =(\frac{\partial(\gamma_{0}I_{3}-I_{1})}{\gamma_{0}\partial R}-\zeta\gamma_{0})/\frac{\partial I_{2}}{2\partial R}\\
 & =\frac{\partial(I_{3}-I_{1})}{\partial R}/\frac{\partial I_{2}}{2\partial R}+(1-\frac{1}{\gamma_{0}})\frac{\partial I_{1}}{\partial R}/\frac{\partial I_{2}}{2\partial R}-\zeta\gamma_{0}/\frac{\partial I_{2}}{2\partial R}\\
 & \geq\frac{\partial(I_{3}-I_{1})}{\partial R}/\frac{\partial I_{2}}{2\partial R}+2\zeta
\end{align*}
where the inequality sign holds because $(1-\frac{1}{\gamma_{0}})\frac{\partial I_{1}}{\partial R}/\frac{\partial I_{2}}{2\partial R}$ is positive. Note that $\frac{\partial(I_{3}-I_{1})}{\partial R}/\frac{\partial I_{2}}{2\partial R}$ is also defined as maximum learning rates of WN, and vanilla SGD in \cite{WNdynamic}. Hence, we conclude that $\eta_{\mathrm{max}}^{\bn}\geq \eta_{\mathrm{max}}^{\{\wn,\ord\}}+2\zeta$.
\end{proof}

\subsection{\textcolor{black}{Proofs regarding generalization and statistical
mechanics (SM)}}
\label{app:sm_proof}
In this section, we build an analytical model for the generalization
ability of a single-layer network. The framework is based on the Teacher-Student
model, where the teacher network output $y^{\ast}=g^{\ast}\left(\mathbf{w^{\ast}}\tran\cdot\mathbf{x}+s\right)$
is learned by a student network. The weight parameter of the teacher
network satisfies $\frac{1}{N}\left(\mathbf{w}^{*}\right)\tran\cdot\mathbf{w}^{\ast}$=1
and the bias term $s$ is a random variable $s\sim\mathcal{N}\left(0,S\right)$
fixed for each training example $\mathbf{x}$ to represent static
errors in training data from observations. In the generalization analysis,
the input is assumed to be drawn from $\mathbf{x}\sim\mathcal{N}\left(0,\frac{1}{N}\mathbf{I}\right)$.
The output of the student can also be written as a similar form $y=g\left(\widetilde{\mathbf{w}}\cdot\mathbf{x}\right)$,
where the activation function $g\left(\cdot\right)$ can be either
linear or ReLU in the analysis and $\widetilde{\mathbf{w}}$ is a
general weight parameter which can be used in either WN or common
linear perceptrons. Here we take WN for example, since it has been
derived in this study that BN can be decomposed into WN with a regularization
term on $\gamma.$ In WN $\widetilde{\mathbf{w}}=\gamma\frac{\mathbf{w}}{\Vert\mathbf{w}\Vert_{2}}$
and we defined the same order parameter as the previous section that
$\gamma^{2}=\frac{1}{N}\widetilde{\mathbf{w}}\tran\cdot\widetilde{\mathbf{w}}$
and $\gamma R=\frac{1}{N}\widetilde{\mathbf{w}}\tran\cdot\mathbf{w}^{\ast}$ .

\subsubsection{Generalization error}\label{sub:gen}

Since the learning task is a regression problem with teacher output
biased by a Gaussian noise, it comes natural that we can use the the
average mean square error loss $\epsilon_{t}=\frac{1}{P}\sum_{j}\left(y_{j}^{*}-y_{j}\right)^{2}$
for the regression. The generalization error defined as the estimation
over the distribution of input $\mathbf{x}$ and is written as

\begin{equation}
\epsilon_{\mathrm{gen}}(\widetilde{\mathbf{w}})=\left\langle \left(y^{*}-y\right)^{2}\right\rangle _{\mathbf{x}}
\end{equation}

where $\left\langle \cdot\right\rangle _{\mathbf{x}}$ denotes an
average over the distribution over $\mathbf{x}$. The generalization
error is a function of its weight parameter and can be converted to
a function only with regard to the aformentioned order parameters,
detailed derivation can be seen in \citep{bos_statistical_1998,krogh_generalization_1992}.

\begin{equation}
\epsilon_{\mathrm{gen}}(\gamma,R)=\iint Dh_{1}Dh_{2}\left[g^{\ast}(h_{1})-g(\gamma Rh_{1}+\gamma\sqrt{1-R^{2}}h_{2})\right]^{2}
\end{equation}

where $h_{1}$and $h_{2}$ are variables drawn from standard Gaussian
distribution and $Dh_{1}:=\mathcal{N}\left(0,1\right)dh_{1}$.

When both the teacher network and student network have a linear activation
function, the above integration can be easily solved and

\begin{equation}
\epsilon_{\mathrm{gen}}(\gamma,R)=1+\gamma^{2}-2\gamma R
\end{equation}

As for the case where the teacher network is linear and the student
network has a ReLU activation, it can still be solved first by decomposing
the loss function

\begin{equation*}
\begin{aligned}\epsilon_{\mathrm{gen}}(\gamma,R) & =\iint Dh_{1}Dh_{2}\left[h_{1}-g(\gamma Rh_{1}+\gamma\sqrt{1-R^{2}}h_{2})\right]^{2}\\
 & =\iint Dh_{1}Dh_{2}\left[h_{1}^{2}+g(\gamma Rh_{1}+\gamma\sqrt{1-R^{2}}h_{2})^{2}-2h_{1}g(\gamma Rh_{1}+\gamma\sqrt{1-R^{2}}h_{2})\right]^{2}\\
 & =1+\frac{\gamma^{2}}{2}-2\iint Dh_{1}Dh_{2}\left[h_{1}g(\gamma Rh_{1}+\gamma\sqrt{1-R^{2}}h_{2})\right]^{2}
\end{aligned}
\end{equation*}

It should be noted that the last two terms should only be integrated
over the half space $\gamma Rh_{1}+\gamma\sqrt{1-R^{2}}h_{2}>0$,
and therefore if we define the angle of this line with the $h_{2}$
axis $\theta_{0}=\arccos\left(R\right)$ the integration is transformed
to polar coordinate

\begin{equation*}
\begin{aligned}\epsilon_{\mathrm{gen}}(\gamma,R) & =1+\frac{\gamma^{2}}{2}-2\iint Dh_{1}Dh_{2}\left[h_{1}g(\gamma Rh_{1}+\gamma\sqrt{1-R^{2}}h_{2})\right]^{2}\\
 & =1+\frac{\gamma^{2}}{2}-2\int_{-\theta_{0}}^{\pi-\theta_{0}}d\theta\int_{0}^{\infty}rdr\frac{1}{2\pi}\exp(-\frac{r^{2}}{2})\left(\gamma Rr^{2}\sin^{2}(\theta)+\gamma\sqrt{1-R^{2}}r^{2}\cos\left(\theta\right)\sin\left(\theta\right)\right)\\
 & =1+\frac{\gamma^{2}}{2}-\gamma R
\end{aligned}
\end{equation*}

\subsubsection{Equilibrium order parameters}\label{sub:equi_order}

Following studies on statistical mechanics, the learning process of
a neural network resembles a Langevin process \citep{mandt_stochastic_2017}
and at the equilibrium the network parameters $\theta$ follow a Gibbs
distribution. That is, the weight vector that yields lower training
error produces higher probability. We have $p(\theta)=Z^{-1}\exp\{-\beta\epsilon_{t}(\theta;\x)\}$,
where $\beta=1/T$ and $T$ is temperature, representing the variance
of noise during training and implicitly controlling the learning process.
$\epsilon_{t}(\theta;\x)$ is an energy term of the training loss
function, $Z=\int d\mathcal{P}(\theta)\exp\{-\epsilon_{t}(\theta;\x)/T\}$
is the partition function, and $\mathcal{P}(\theta)$ is a prior distribution.
%of $\theta$.
% irrespective of the temperature.

% And the at equilibrium\footnote{$p\left(\mathbf{w}\right)$ is short for $\prod_{i=1}^{N}p\left(w^{i}\right)$}.
%\begin{equation}
%p\left(\mathbf{w}\right)=\exp\left[-E_{t}(\mathbf{w})/T\right]/Z\label{eq:boltzmann}
%\end{equation}
%where $T$ is temperature, $E\left(\mathbf{w}\right)$ is the energy
%term corresponding to Eqn (\ref{eq:tra}) which is the training error
%in a network, $Z=\int d\mathcal{P}\left(\mathbf{w}\right)\exp\left[-E\left(\mathbf{w}\right)/T\right]$
%is the partition function and $\mathcal{P}\left(\mathbf{w}\right)$
%is a prior distribution of $\mathbf{w}$ irrespective of the temperature.

Instead of directly minimizing the energy term above, statistical
mechanics finds the minima of free energy, $f$, which is a function
over $T$, considering the fluctuations of $\theta$ at finite temperatures.
We have $-\beta f=\langle\ln Z\rangle_{\x}$. %\begin{equation}
%\end{equation}

By substituting the parameters that minimize $f$ back into the generalization
errors calculated above, we are able to calculate the averaged generalization
error, at a certain temperature.

The solution of SM requires the differentiation of $f$ with respect
to the order parameters.

In general, the expression of free energy under the replica theory
is expressed as\citep{seung_statistical_1992}
\begin{equation}\label{eq:free_complete}
-\beta f=\frac{1}{2}\frac{(\gamma^{2}-\gamma^{2}R^{2})}{q^{2}-\gamma^{2}}+\frac{1}{2}\ln(q^{2}-\gamma^{2})+\alpha\iint Dh_{1}Dh_{2}\ln\left[\int Dh_{3}\exp\left(-\frac{\beta\left(g-g^{\ast}\right)^{2}}{2}\right)\right]
\end{equation}
where
\begin{equation}
\begin{aligned}g:=g\left(\gamma Rh_{1}+\sqrt{\gamma^{2}-\gamma^{2}R^{2}}h_{2}+\sqrt{q^{2}-\gamma^{2}}h_{3}\right)\\
g^{\ast}:=g^{\ast}(h_{1}+s)
\end{aligned}
\end{equation}

In the above expression, $h_{1},h_{2},h_{3}$ are three independent
variables following the standard Gaussian distribution and $\alpha=P/N$
represents the ratio of the number of training samples $P$ to number
of unknown parameters $N$ in the network, $R=\frac{1}{N}\frac{\w}{\Vert\w\Vert_2}\cdot\w^\ast$, $q$ is the prior value of $\gamma$..

The above equation can be utilized for a general SM solution of a
network. However, the solution is notoriously difficult to solve and
only a few linear settings for the student network have close-form
solutions\citep{bos_statistical_1998}. Here we extend the previous
analysis of linear activations to a non-linear one, though still under
the condition that $\beta\rightarrow\infty$, which means that the
student network undergoes a exhaustive learning that minimizes the
training error. In the current setting, the student network is a nonlinear
ReLU network while the teacher is a noise-corrupted linear one.

\begin{prop} \label{prop:free_energy} Given a single-layer linear
teacher $y^{\ast}=\w^{\ast}\x+s$ and a student ReLU network $y=g(\gamma\frac{\w}{\Vert\w\Vert_{2}}\x)$
linear student network with $g$ being a ReLU activation function, $\x\sim\mathcal{N}(0, \frac{\mathbf{I}}{N})$
the free energy $f$ satisfies as $\beta\rightarrow\infty$
\begin{equation}
\begin{aligned}-\beta f & =\frac{1}{2}\frac{(\gamma^{2}-\gamma^{2}R^{2})}{q^{2}-\gamma^{2}}+\frac{1}{2}\ln(q^{2}-\gamma^{2})\\
 & -\frac{\alpha}{4}\ln\left(1+\beta\left(q^{2}-\gamma^{2}\right)\right)-\frac{\alpha\beta\left(1-2\gamma R+\gamma^{2}+S\right)}{4\left(1+\beta\left(q^{2}-\gamma^{2}\right)\right)}-\frac{\alpha\beta}{4}-\frac{\alpha\beta}{4}S
\end{aligned}
\end{equation}
where $S$ is the variance of the Gaussian noise $s$ injected to the output of the teacher.
\end{prop}

\begin{proof}
The most difficult process in Eqn.\ref{eq:free_complete}
is to solve the inner integration over $h_{3}$. Here as $\beta\rightarrow\infty$,
it is noted that the function $\exp\left(-\beta x\right)$ only notches
up only at $x=0$ and is 0 elsewhere. Therefore, the integration $\int Dh_{3}\exp\left(-\frac{\beta\left(g-g^{\ast}\right)^{2}}{2}\right)$
depends on the value of $g^{*}$. If $g^{*}<0$, no solution exists
for $g-g^{\ast}=0$ as $g$ is a ReLU activation, and thus the integration
is equivalent to the maximum value of the integral under the limit
of $\beta\rightarrow\infty$. As $g^{*}>0$, the integration over
the ``notch'' is equivalent to the one at full range. That is,

\begin{equation*}
\int Dh_{3}\exp\left(-\frac{\beta\left(g-g^{\ast}\right)^{2}}{2}\right)=\begin{cases}
\int Dh_{3}\exp\left(-\frac{\beta\left(g-g^{\ast}\right)^{2}}{2}\right) & h_{1}+s>0\\
\max_{h_{3}}\exp\left(-\frac{\beta\left(g-g^{\ast}\right)^{2}}{2}\right) & h_{1}+s\leq0
\end{cases}
\end{equation*}

The above equation can be readity integrated out and we obtain

\begin{equation*}
\begin{aligned}\ln\int Dh_{3}\exp\left(-\frac{\beta\left(g-g^{\ast}\right)^{2}}{2}\right) & =-\frac{1}{2}\ln\left(1+\beta\left(q^{2}-\gamma^{2}\right)\right)\\
 & -\frac{\beta}{2}\frac{\left(\left(1-\gamma R\right)h_{1}-\sqrt{\gamma^{2}-\gamma^{2}R^{2}}h_{2}+s\right)^{2}}{1+\beta\left(q^{2}-\gamma^{2}\right)}
\end{aligned}
\end{equation*}

Substituting it back to Eqn.\ref{eq:free_complete}, we have its third
term equivalent

\begin{equation*}
\begin{aligned}\textrm{Term3} & =\alpha\iint_{h_{1}+s>0}Dh_{1}Dh_{2}\left[-\frac{\beta}{2}\frac{\left(\left(1-\gamma R\right)h_{1}-\sqrt{\gamma^{2}-\gamma^{2}R^{2}}h_{2}+s\right)^{2}}{1+\beta\left(q^{2}-\gamma^{2}\right)}\right]\\
 & =\alpha\iint_{h_{1}+s>0}Dh_{1}Dh_{2}\left[-\frac{\beta}{2}\frac{\left(\left(1-\gamma R\right)^{2}h_{1}^{2}-\sqrt{\gamma^{2}-\gamma^{2}R^{2}}h_{2}+s\right)^{2}}{1+\beta\left(q^{2}-\gamma^{2}\right)}\right]
\end{aligned}
\end{equation*}

To solve the above integration, we first realize that $s$ is a random
variable to corrupt the output of the teacher output and the above
integration should be averaged out over $s$. Given that $s\sim\mathcal{N}\left(0,S\right)$
, it is easy to realize that

\begin{equation*}
\left\langle \int_{h+s>0}s^{2}Dh\right\rangle _{s}=\frac{S}{2},\text{ }\left\langle \int_{h+s>0}Dh\right\rangle _{s}=\frac{1}{2},\text{ and }\left\langle \int_{h+s>0}hsDh\right\rangle _{s}=0
\end{equation*}

Through simple Gaussian integraions, we get

\begin{equation*}
\begin{aligned}\textrm{Term3} & =\alpha\left[-\frac{1}{4}\ln\left(1+\beta\left(q^{2}-\gamma^{2}\right)\right)-\frac{\beta\left(1-2\gamma R+\gamma^{2}+S\right)}{4\left(1+\beta\left(q^{2}-\gamma^{2}\right)\right)}-\frac{\beta}{4}-\frac{\beta}{4}S\right]\end{aligned}
\end{equation*}

Substituting Term3 back yields the results of the free energy.
\end{proof}

Therefore, by locating the values that minimizes $f$ in the above proposition, we have equilibrium order parameters
\begin{equation}
\gamma^{2}=\frac{\alpha}{2a}+\frac{\alpha S}{2a-\alpha}
\end{equation}
and
\begin{equation}
\gamma R=\frac{\alpha}{2a}
\end{equation}

where $a$ is defined as $a=\frac{1+\beta\left(q^{2}-\gamma^{2}\right)}{\beta\left(q^{2}-\gamma^{2}\right)}$.
Substituting the order parameters back to the generalization error,
we have

\begin{equation}
\epsilon_{\mathrm{gen}}=1-\frac{\alpha}{4a}+\frac{\alpha S}{2a\left(2a-\alpha\right)}
\end{equation}

When $\alpha<2$ and $\beta\rightarrow\infty$, $a=1$, the generalization
error is
\begin{equation} \label{eq:relu_gen}
\epsilon_{\mathrm{gen}}=1-\frac{\alpha}{4}+\frac{\alpha S}{2\left(2-\alpha\right)}
\end{equation}

\end{document}